\documentclass[letterpaper,journal,twoside]{IEEEtran} 
\usepackage{ifthen} 
\def\useieeetemplate{1}
\def\arxivversion{1}

\usepackage{amsmath,amsfonts}
\usepackage{cite}
\usepackage{threeparttable}     
\usepackage[caption=false,font=normalsize,labelfont=sf,textfont=sf]{subfig}
\usepackage{textcomp}           
\usepackage[utf8]{inputenc}     
\usepackage[T1]{fontenc}        
\usepackage{graphicx} 
\usepackage{url}
\usepackage{hyperref}

\usepackage{multirow}
\usepackage{algorithmic}
\usepackage{algorithm}
\usepackage{booktabs}
\usepackage{pgfplots}           
\pgfplotsset{compat=1.18}       
\pgfplotsset{
    ybar legend/.style={
        /pgfplots/legend image code/.code={
            \fill[##1] (0cm,0.6em) rectangle (\pgfplotbarwidth,-0.3em);
}, },
}
\usepackage{pgfplotstable}      
\usepackage{tikz}
\usetikzlibrary{arrows.meta,positioning}
\newcommand{\flowedge}[5]{
    \draw [myarrow, #4] (#1) -- node[fill=white, pos=0.45, inner sep=1pt, #5] {#3}  (#2);
}
\newcommand{\shadeedge}[2]{
    \draw [line width=5 pt, green!60!black, opacity=0.2] (#1) -- (#2);
}
\newcommand{\shadebentedge}[2]{
    \draw [line width=5 pt, green!60!black, opacity=0.2] (#1) to[bend left=10] node {} (#2);
}
\newcommand{\doubleflowedge}[7]{
    \draw [myarrow, #5] (#1) to[bend left=10] node[fill=white, pos=0.45, inner sep=1pt, #7] {#3} (#2);
    \draw [myarrow, #6] (#2) to[bend left=10] node[fill=white, pos=0.45, inner sep=1pt, #7] {#4} (#1);
}

\newtheorem{lemma}{Lemma}

\newtheorem{corollary}{Corollary}

\title{Fast and Compact Graph Cuts for the Boykov-Kolmogorov Algorithm}

\ifthenelse{\useieeetemplate=1}{
\ifthenelse{\arxivversion=1}{
\usepackage{orcidlink}

\author{Christian~M.~Mikkelstrup\,\orcidlink{0009-0008-4765-151X}, Anders~B.~Dahl\,\orcidlink{0000-0002-0068-8170}, Philip~Bille\,\orcidlink{0000-0002-1120-5154}, Vedrana~A.~Dahl\,\orcidlink{0000-0001-6734-5570}, and Inge~Li~Gørtz\,\orcidlink{0000-0002-8322-4952}
\thanks{This work was supported by Danish Data Science Academy, which is funded by the Novo Nordisk Foundation (NNF21SA0069429) and VILLUM FONDEN (40516). \textit{(Corresponding author: Christian M. Mikkelstrup).}}%
\thanks{The authors are with DTU Compute, Technical University of Denmark, 2800 Kongens Lyngby, Denmark (e-mail: \{cmomi; abda; phbi; vand; inge\}@dtu.dk).}
\thanks{This work has been submitted to the IEEE for possible publication. Copyright may be transferred without notice, after which this version may no longer be accessible.}}

\markboth{Mikkelstrup \MakeLowercase{\textit{et al.}}: Fast and Compact Graph Cuts for the Boykov-Kolmogorov Algorithm}%
{Mikkelstrup \MakeLowercase{\textit{et al.}}: Fast and Compact Graph Cuts for the Boykov-Kolmogorov Algorithm}
\IEEEpubid{}
}{
\author{Christian~M.~Mikkelstrup, Anders~B.~Dahl,~\IEEEmembership{Member,~IEEE,
} Philip~Bille, Vedrana~A.~Dahl,~\IEEEmembership{Member,~IEEE,
} and Inge~Li~Gørtz
\thanks{Manuscript received May 7, 2026; revised October 16, 2026. This work was supported by Danish Data Science Academy, which is funded by the Novo Nordisk Foundation (NNF21SA0069429) and VILLUM FONDEN (40516). \textit{(Corresponding author: Christian M. Mikkelstrup).}}%
\thanks{The authors are with DTU Compute, Technical University of Denmark, 2800 Kongens Lyngby, Denmark (e-mail: \{cmomi; abda; phbi; vand; inge\}@dtu.dk).}
\thanks{Digital Object Identifier no 00.0000/TPAMI.2026.0000000}}

\markboth{IEEE Transactions on Pattern Analysis and Machine Intelligence,~Vol.~48, No.~11, November~2026}%
{Mikkelstrup \MakeLowercase{\textit{et al.}}: Fast and Compact Graph Cuts for the Boykov-Kolmogorov Algorithm}
\IEEEpubid{0000--0000/00\$00.00~\copyright~2026 IEEE}
}
}{
\author{Christian Møller Mikkelstrup \\ \texttt{cmomi@dtu.dk} \and Anders B. Dahl \\ \texttt{abda@dtu.dk} \and Philip Bille \\ \texttt{phbi@dtu.dk} \and Vedrana A. Dahl \\ \texttt{vand@dtu.dk} \and Inge Li Gørtz \\ \texttt{inge@dtu.dk}}
}

\date{}

\begin{document}
\maketitle

\ifthenelse{\useieeetemplate=1}{
\begin{abstract}
Computing a minimum $s$-$t$ cut in a graph is a solution to a wide range of computer vision problems and is often done using the Boykov-Kolmogorov (BK) algorithm. In this paper, we revisit the BK algorithm from both a theoretical and practical point of view. We improve the analysis of the time complexity of the BK algorithm to $O(mn|C|)$ and propose a new algorithm, the fast BK (fBK) algorithm, with a time complexity of $O(m|C|)$, where $m$, $n$, and $|C|$ are the number of edges, the number of vertices, and the capacity of the cut, respectively. We additionally propose a compact graph representation that allows our implementations of the BK and fBK algorithms to find a minimum $s$-$t$ cut in a graph with up to $10^9$ vertices and $10^{10}$ edges on a machine with 128~GB of memory. We denote these as \texttt{cBK} and \texttt{fcBK}, respectively. Despite \texttt{fcBK} having the better worst-case guarantee, we find \texttt{cBK} to be the fastest available implementation within the family of BK algorithms when evaluating on a comprehensive set of benchmark datasets, highlighting the importance of memory-efficient implementations. We make our implementations publicly available for further research and implementation development within minimum $s$-$t$ cut algorithms. 
\end{abstract}
}{}

\section{Introduction}
\ifthenelse{\useieeetemplate=1}{\IEEEPARstart{G}{iven}}{Given} a weighted, directed graph $G = (V,E)$ and two \emph{terminal vertices} $s$ and $t$, an \emph{$s$-$t$ cut} is a partition of $V$ into sets $A$ and $B$ containing $s$ and $t$, respectively, and the \emph{capacity} of the cut is the sum of weights of the edges from $A$ to $B$. The \emph{minimum $s$-$t$ cut problem} is to compute an $s$-$t$ cut of minimum capacity.

In this paper, we revisit the well-known \emph{Boykov-Kolmogorov} (BK) algorithm \cite{boykov2004a} for computing minimum $s$-$t$ cuts in graphs. The BK algorithm is widely used in computer vision applications, including image segmentation \cite{boykov2001a, soc2005a, boykov2006a, delong2012a}, surface detection \cite{li2004a, tran2006a, li2006a}, multi-view reconstruction \cite{kolmogorov2002multi, labatut2007a}, and stereo matching \cite{kolmogorov2014kolmogorov}. The BK algorithm is also used in complex vision pipelines, such as in prediction steps~\cite{silvoster2025a, li2019a, borovec2017a}, in large frameworks~\cite{cristinelli2026a, zhang2019a, zhou2023a, barath2018a, jain2024a, devi2023a}, and in combination with deep learning~\cite{tang2025a, xie2025a, quo2018a}.

The BK algorithm alternates three stages called \emph{growth}, \emph{augmentation}, and \emph{adoption} to first compute the \emph{maximum $s$-$t$ flow} in the graph. During the growth stage, the algorithm adds edges to trees rooted at $s$ and $t$ until it discovers an edge connecting the two trees. This edge is then used by the augmentation stage for traversing an \emph{augmenting path} from $s$ to $t$, along which it can push flow. This push might remove edges from the trees, creating trees rooted at non-terminal vertices, called \emph{orphans}. The adoption stage reattaches orphans to vertices that are still connected to the terminals. This restores the trees (rooted at $s$ and $t$) for the next growth stage. Once no more paths can be found, the algorithm has a maximum $s$-$t$ flow and the two trees induce a minimum $s$-$t$ cut. The time complexity of the BK algorithm is $O(mn^2|C|)$, where $n$ is the number of vertices, $m$ is the number of edges, and $|C|$ is the capacity of the minimum $s$-$t$ cut in the graph.

The BK algorithm has a list of heuristics designed to improve the experimental performance of the algorithm when applied to graphs from computer vision. For these graph instances, the algorithm often runs much faster in practice than its worst-case bound suggests. In addition to the original implementation, the BK algorithm has been implemented as variants with reduced memory footprint \cite{goldberg2011a, goldberg2015a, jensen2023a} and improved performance on specific classes of graphs \cite{li2004a, jamriska2012a}.

In this paper, we consider the BK algorithm \cite{boykov2004a} and its variants \cite{jensen2023a} from both a theoretical and practical point of view. We focus on the time and memory required for graphs from computer vision and propose several improvements. Our contributions are:
\ifthenelse{\useieeetemplate=1}{\IEEEpubidadjcol}{}
\begin{itemize}
    \item We improve the analysis of the time complexity of the BK algorithm to $O(mn|C|)$. Compared to the existing theoretical analysis, this improves the running time by a factor of $n$. We also show that a constant upper bound on the weight of each non-terminal edge (similarly for terminal edges) further improves the time complexity of the BK algorithm to $O(m^2n)$. This upper bound on the edge weight is often realized in computer vision.
    \item We introduce the \emph{fast Boykov-Kolmogorov} (fBK) algorithm, which is our variant of the BK algorithm. By flagging encountered vertices in orphan-rooted trees, we remove repeated work in the adoption stage. The repeated work in the BK algorithm comes from multiple vertices using the same tree path to check whether they are connected to a terminal. The fBK algorithm has a time complexity of $O(m|C|)$, improving over the BK algorithm by a factor of $n$. As for the BK algorithm, this is further improved by a constant upper bound on the weight of each non-terminal edge (similarly for terminal edges), yielding a time complexity of $O(m^2)$ for the fBK algorithm.
    \item We propose a compact graph representation that significantly improves the graph footprint and peak memory usage when using any variant of BK to find the minimum $s$-$t$ cut. The representation does this for general computer vision instances using techniques such as interleaving vertices and edges, merging edges with the same source and target, letting multiple values occupy the same memory locations, and storing index offsets relative to the current vertex. Additionally, we design the preprocessing such that we can construct this representation from a graph with little memory overhead. We verify that our representation has low peak memory consumption, which is required for practical scalability.
    \item We create efficient \texttt{C++} implementations of the BK and fBK algorithms, both using our compact graph representation. We denote these as \texttt{cBK} and \texttt{fcBK}, respectively. We make our implementation publicly available at \url{https://fcbk-maxflow.github.io/}.
    \item We experimentally evaluate our implementations against existing implementations of BK~\cite{boykov2004a, jensen2023a} on a comprehensive set of 61 benchmark datasets. The experiments show that, on average, our implementations use less than half of the memory and run faster than existing BK implementations. Compared to the original, both \texttt{cBK} and \texttt{fcBK} have a shorter running time to find the maximum $s$-$t$ flow on every dataset while using between $3.5\%$ and $51\%$ of the memory. 
    \item We investigate the practical impact of our orphan-path flags on computer vision datasets. We do this by experimentally analyzing the work done during the adoption stage of \texttt{cBK} and \texttt{fcBK}. In contrast to the theoretical analysis, we find that the repeated work has limited contribution to the practical running time. Our experiments highlight our \texttt{cBK} implementation as the fastest practical BK implementation available for the considered datasets.
    \item We evaluate our \texttt{cBK} and \texttt{fcBK} implementations on two examples of large surface detection problems. By utilizing the compact graph representation, we find the minimum $s$-$t$ cut in graphs with up to $10^9$ vertices and $10^{10}$ edges on a machine with $128$~GB of memory. These examples highlight that our compact graph representation is the preferred option under memory constraints for the BK algorithm when solving large minimum $s$-$t$ cuts in computer vision problems.
\end{itemize}

\subsection{Related Work}

Approaches that reduce the memory usage of minimum $s$-$t$ cut algorithms may target specific use cases and exploit their induced structure. Examples include grid graphs \cite{jamriska2012a}, surface detection \cite{li2004a}, and planar graphs \cite{schmidt2009a}. However, such implementations cannot be used for solving graph instances that deviate from the assumed structure (e.g.,\ \cite{jeppesen2020a, liu2015a, rother2007a}).

Several improved theoretical bounds for the minimum $s$-$t$ cut have been shown. This includes algorithms using augmenting paths \cite{goldberg2011a}, pseudoflows \cite{hochbaum2008a, hochbaum2013a, goldberg2015a}, and variants of the push-relabel algorithm \cite{goldberg1988a, cherkassky1997a, arora2010a, goldberg2008a, goldberg2009a}. Many of these have efficient practical implementations, but are not competitive for large-scale computer vision problem instances. Other related lines of research include approximate algorithms~\cite{lerm2010a, isack2017a, jeppesen2020a} and parallel algorithms~\cite{liu2009a, liu2010a, ekstroem2020a, bader2005a, delong2008a}.    

\subsection{Section Overview}

Section~\ref{sec:prelim} covers preliminaries, including the BK algorithm. Section~\ref{sec:bk_runtime} presents our improved time‑complexity analysis of the BK algorithm. Sections~\ref{sec:compact_graph} and \ref{sec:fcbk} introduce our compact graph representation and our fBK algorithm. Sections~\ref{sec:benchmark} and \ref{sec:largevolumes} report experiments, and Section~\ref{sec:discussion_conclusion} concludes.

\section{Preliminaries}\label{sec:prelim}
We introduce notation for minimum $s$-$t$ cuts, maximum $s$-$t$ flows (visualized in \figurename~\ref{fig:exGandGf}), and the BK algorithm (visualized in \figurename~\ref{fig:exBK}) to set up our new fBK algorithm.

\subsection{Minimum \texorpdfstring{$s$-$t$}{s-t} Cut and Maximum \texorpdfstring{$s$-$t$}{s-t} Flow}

In the minimum $s$-$t$ cut problem, we are given a connected, directed, and edge-weighted graph $G=(V,E)$. The vertex set $V$ has two special vertices $s,t$, called \emph{terminals}. The terminal $s\in V$ has only outgoing edges and the terminal $t\in V$ has only incoming edges. An edge is called a \emph{terminal edge} if it is incident with a terminal and an \emph{internal edge} otherwise. Each edge $(u,v)$ in the edge set $E$ is associated with a non-negative value $c(u,v)\geq 0$, called the \emph{capacity of the edge}. A \emph{cut} in $G$ is a partition of $V$ into two sets, denoted $A$ and $B$. We require that $s\in A$ and $t\in B$. A cut $(A,B)$ is associated with a value $c(A,B)$, called the \emph{capacity of the cut}. The capacity of a cut is the sum of the capacities of all edges going from a vertex in $A$ to a vertex in $B$. The minimum $s$-$t$ cut problem is finding a cut $(A,B)$ with minimum capacity $c(A,B)=|C|$.

In the maximum $s$-$t$ flow problem, each edge $(u,v)\in E$ is additionally associated with a value, $f(u,v)$, called the \emph{flow of the edge}. The \emph{flow of a graph} is an assignment of flow to all edges in $E$, denoted by $f$. We require that the flow of a graph satisfies the capacity and conservation constraints. The \emph{capacity constraint} requires that every edge $(u,v)\in E$ satisfy $0\leq f(u,v)\leq c(u,v)$. The \emph{conservation constraint} requires that every non-terminal vertex $u\in V\backslash\{s,t\}$ satisfy $\sum_{(u,v)\in E}f(u,v)=\sum_{(v,u)\in E}f(v,u)$. The \emph{value of a flow} $f$ is the sum of the flow leaving $s$ (or equivalently the flow entering $t$) and is denoted by $|f|$. 

Most algorithms for computing the minimum $s$-$t$ cut start by finding the maximum $s$-$t$ flow in the graph.

A graph $G$ is represented with a flow $f$ using a directed and weighted \emph{residual graph}, denoted $G_f$. The residual graph $G_f$ contains the same vertex set as $G$ and uses \emph{residual edges}, denoted $E_f$. For now, we assume that any two vertices in $G$ have no more than one edge between them. For every edge in the original graph $(u,v)\in E$ with capacity $c(u,v)$ and flow $f(u,v)$, the residual graph has a pair of residual edges $(u,v)\in E_f$ and $(v,u)\in E_f$. From this pair, the residual edge $(u,v)\in E_f$ (resp. $(v,u)\in E_f$) is associated with a non-negative value $c_f(u,v)=c(u,v)-f(u,v)\geq 0$ (resp. $c_f(v,u)=f(u,v)\geq 0$), called the \emph{residual capacity of the edge}. For any such pair of residual edges $e, e'\in E_f$, we say that $e'$ is the \emph{mirror} of $e$, and vice versa. We say that a residual edge $(u,v)\in E_f$ is \emph{saturated} if and only if $c_f(u,v)=0$. If there are multiple edges between two vertices $u,v\in G$, there are different representations of the residual graph $G_f$. Possible representations include a unique pair of residual edges for each edge between $u$ and $v$, and a single pair of residual edges $(u,v),(v,u)\in E_f$.

An \emph{augmenting path} is a simple path in $G_f$ from $s$ to $t$ that uses non-saturated residual edges. The \emph{bottleneck capacity} of a path, denoted by $\Delta$, is the smallest residual capacity of a residual edge on the path. Given an augmenting path and the corresponding bottleneck capacity $\Delta$, a legal operation is to decrease the residual capacity for edges in the augmenting path by $\Delta$ and increase the residual capacity for their mirror edges by $\Delta$. We say that this operation \emph{pushes} $\Delta$ flow from $s$ to $t$. Pushing $\Delta$ flow along an augmenting path will increase the value of the flow $|f|$ by $\Delta$ and saturate at least one residual edge along the path. 

An \emph{augmenting path algorithm} is an algorithm for finding the maximum $s$-$t$ flow in a graph $G$ that operates on the residual graph $G_f$ with an initially empty flow $f$. It repeatedly finds augmenting paths from $s$ to $t$ and pushes the bottleneck capacity along them. The algorithm terminates once no augmenting paths exist. At this point, the algorithm has a maximum value flow of $|f|=|C|$. We can then find a minimum cut $(A,B)$ by defining $A$ as the vertices reachable in $G_f$ from $s$ via non-saturated residual edges, and $B$ the remaining vertices. The differences between maximum $s$-$t$ flow algorithms that use augmenting paths lie in how they find the augmenting paths. When using an augmenting path algorithm to find the minimum $s$-$t$ cut, the running time is dominated by computing the maximum $s$-$t$ flow. Because all algorithms for finding a minimum $s$-$t$ cut return a cut of the same capacity, we evaluate the performance of an algorithm by the time and memory it requires.

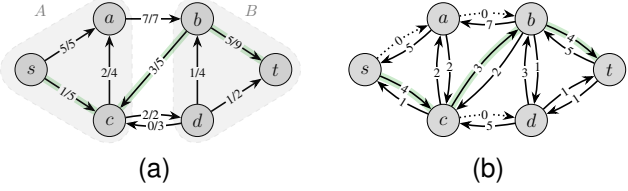
\begin{figure}[t]
    \centering
    \subfloat[]{
        \noindent\resizebox{0.23\textwidth}{!}{\begin{tikzpicture}[
        stdnode/.style={
            circle,
            draw=black,
            fill=gray,
            fill opacity = 0.3,
            text opacity=1,
            inner sep=0pt,
            minimum size=20pt,
            font=\large 
        },
        myarrow/.style={
            -Stealth,
            font=\footnotesize, 
            line width=1pt
        },
        node distance=0.6cm and 1.2cm,
    ]
    \useasboundingbox (-0.75,-1.625) rectangle (6, 1.625);
    
    \node[stdnode] (s) {$s$};
    \node[stdnode, above right=of s] (v1) {$a$};
    \node[stdnode, below right=of s] (v2) {$c$};
    \node[stdnode, right=of v1] (v3) {$b$};
    \node[stdnode, right=of v2] (v4) {$d$};
    \node[stdnode, below right=of v3] (t) {$t$};

    \flowedge{s}{v1}{5/5}{black}{sloped}
    \shadeedge{s}{v2}
    \flowedge{s}{v2}{1/5}{black}{sloped}
    \flowedge{v2}{v1}{2/4}{black}{}
    \flowedge{v1}{v3}{7/7}{black}{}
    \doubleflowedge{v2}{v4}{2/2}{0/3}{black}{black}{}
    \shadeedge{v3}{v2}
    \flowedge{v3}{v2}{3/5}{black}{sloped}
    \flowedge{v4}{v3}{1/4}{black}{}
    \shadeedge{v3}{t}
    \flowedge{v3}{t}{5/9}{black}{sloped}
    \flowedge{v4}{t}{1/2}{black}{sloped}

    \foreach \i/\x/\y in {0/-1.2/0,1/2.2/2,2/2.2/-2,3/6.4/0,4/3.1/2,5/3.1/-2}
        \coordinate (c\i) at (\x,\y) {};
    \draw[dashed, rounded corners=30, black, fill=gray, opacity=0.1] (c0)--(c1)--(c2)--cycle;
    \draw[dashed, rounded corners=30, black, fill=gray, opacity=0.1] (c3)--(c4)--(c5)--cycle;
    \node[gray] (labelA) at (0.2,1.3) {$A$};
    \node[gray] (labelB) at (4.8,1.3) {$B$};

\end{tikzpicture}}
    \label{fig:exG}} %
    \subfloat[]{
        \noindent\resizebox{0.23\textwidth}{!}{\begin{tikzpicture}[
        stdnode/.style={
            circle,
            draw=black,
            fill=gray,
            fill opacity = 0.3,
            text opacity=1,
            inner sep=0pt,
            minimum size=20pt,
            font=\large 
        },
        myarrow/.style={
            -Stealth,
            font=\footnotesize, 
            line width=1pt
        },
        node distance=0.6cm and 1.2cm,
    ]
    \useasboundingbox (-0.75,-1.625) rectangle (6, 1.625);%
    
    \node[stdnode] (s) {$s$};
    \node[stdnode, above right=of s] (v1) {$a$};
    \node[stdnode, below right=of s] (v2) {$c$};
    \node[stdnode, right=of v1] (v3) {$b$};
    \node[stdnode, right=of v2] (v4) {$d$};
    \node[stdnode, below right=of v3] (t) {$t$};

    \doubleflowedge{s}{v1}{0}{5}{black, dotted}{black}{sloped}
    \shadebentedge{s}{v2}
    \doubleflowedge{s}{v2}{4}{1}{black}{black}{sloped}
    \doubleflowedge{v2}{v1}{2}{2}{black}{black}{}
    \doubleflowedge{v3}{v1}{7}{0}{black}{black, dotted}{}
    \doubleflowedge{v4}{v2}{5}{0}{black}{black,dotted}{}
    \shadebentedge{v2}{v3}
    \doubleflowedge{v3}{v2}{2}{3}{black}{black}{sloped}
    \doubleflowedge{v4}{v3}{3}{1}{black}{black}{}
    \shadebentedge{v3}{t}
    \doubleflowedge{v3}{t}{4}{5}{black}{black}{sloped}
    \doubleflowedge{v4}{t}{1}{1}{black}{black}{sloped}
    
\end{tikzpicture}}
    \label{fig:exGf}} %
    \caption{Example describing the notation. (a) A graph $G$ and a flow $f$ with $|f|=6$. Each edge $(u,v)\in E$ is labeled with $f(u,v)/c(u,v)$. A cut $(A,B)$ is shown with value $c(A,B)=9$. 
    (b) The residual graph $G_f$ corresponding to $G$ and $f$ from (a). Each edge $(u,v)\in E_f$ is labeled with $c_f(u,v)$. Saturated residual edges are dashed. A single pair of residual edges between $c$ and $d$ represents the corresponding two edges in $G$. 
    Both (a) and (b) have the edges corresponding to an augmenting path in $G_f$ highlighted in green. The bottleneck capacity of this path is $\Delta=3$. If we push $\Delta=3$ flow along this path, we get $|f|=|C|$.}
    \label{fig:exGandGf}
\end{figure}

A \emph{directed tree} is a directed graph where the underlying undirected graph is connected and without cycles. An \emph{out-tree} is a directed tree with edges oriented away from the root. An \emph{in-tree} is a directed tree with edges oriented towards the root. We use \emph{tree} to refer to an in-tree or an out-tree. An \emph{out-forest} is a collection of out-trees. An \emph{in-forest} is a collection of in-trees. Consider a pair of neighboring vertices in a tree. The vertex closest to the root is called the \emph{parent} and the vertex furthest from the root is called the \emph{child}.

\subsection{The Boykov-Kolmogorov Algorithm}\label{sec:bk}

The BK algorithm is an augmenting path algorithm for finding the maximum $s$-$t$ flow. It finds augmenting paths in the residual graph $G_f=(V,E_f)$ from the terminal $s$ to the terminal $t$. It maintains an out-forest called $S$ and an in-forest called $T$. All edges in $S$ and $T$ are non-saturated residual edges. The terminal $s$ is the root of a tree in $S$ and the terminal $t$ is the root of a tree in $T$. If a vertex different from $s$ or $t$ is the root of a tree in $S\cup T$, that vertex is called an \emph{orphan}. A vertex $v\in V$ is called a \emph{tree vertex} if $v\in S\cup T$ and a \emph{free vertex} otherwise. The BK algorithm starts with $s$ and $t$ as tree vertices and all other vertices being free. The algorithm loops through three stages, called \emph{growth}, \emph{augmentation}, and \emph{adoption}. The notation and intuition for the stages are described here and visualized in \figurename~\ref{fig:exBK}, and the pseudocode is detailed in Algorithms \ref{alg:bkgrowth}, \ref{alg:bkaugment}, and \ref{alg:bkadoption}. See \cite{kolmogorov2004b} for the proof of correctness.

During the growth stage, $S$ and $T$ contain only the trees rooted at $s$ and $t$. The stage maintains a set of vertices that it will process later in the algorithm, called \emph{active} vertices. A \emph{passive} vertex is a tree vertex that is not active. The terminals $s$ and $t$ are initially active. The growth stage goes through each active vertex $v$ and its corresponding non-saturated residual edges (outgoing for $v\in S$ and incoming for $v\in T$ to respect the out-/in-forest orientation) with the goal of growing $S$ and $T$. The growth stage stops in one of two cases. If all tree vertices are passive, the algorithm can no longer grow $S$ or $T$ (no augmenting paths exist) and terminates with a maximum $s$-$t$ flow (and a minimum $s$-$t$ cut by defining $A$ to be the vertices of $S$, and $B$ the remaining vertices). The growth stage also ends if it encounters a non-saturated residual edge from a vertex in $S$ to a vertex in $T$, called a \emph{bridge}. In this case, the augmentation stage begins. 

The augmentation stage finds an augmenting path from $s$ to $t$ that crosses the bridge. It then pushes the bottleneck capacity along the path. This push will saturate at least one residual edge in the augmenting path. Any saturated edges in $S$ or $T$ are removed from their forest. This maintains the property that $S$ and $T$ only contain non-saturated residual edges. Removing edges from $S$ and $T$ creates orphans. 

The adoption stage processes all orphans. Each orphan $o$ is either reconnected to a tree rooted at a terminal or made free. If $o\in S$, it must be reconnected to a vertex in $S$ with a non-saturated residual edge to $o$. If $o\in T$, the orphan must be reconnected with a vertex in $T$ to which it has a non-saturated residual edge. We call such a vertex a \emph{potential parent}. A potential parent can reconnect an orphan only if the root of its tree is a terminal. The \texttt{findroot} subroutine follows parent edges from a potential parent to determine the root of its tree. If an orphan does not have a potential parent connected to a terminal, it is made free, its potential parents are made active, and its children are made orphans. At the end of the adoption stage, $S$ is an out-tree rooted at $s$, $T$ is an in-tree rooted at $t$, and the growth stage begins again.

\begin{algorithm}[H]
    \caption{BK stage: growth}
    \label{alg:bkgrowth}
    \begin{algorithmic}
        \STATE \textbf{while} there are active vertices in $S\cup T$ \textbf{do}
        \STATE \hspace{1em} Select an active $v\in S\cup T$
        \STATE \hspace{1em} \textbf{for} each neighbor $u$ of $v$ \textbf{do}
        \STATE \hspace{2em} Let $e$ be $(v,u)$ if $v\in S$ or $(u,v)$ if $v\in T$
        \STATE \hspace{2em} \textbf{if} $e$ is not saturated \textbf{then}
        \STATE \hspace{3em} \textbf{if} $u$ is free \textbf{then}
        \STATE \hspace{4em} Add $u$ to the tree of $v$ using $e$
        \STATE \hspace{4em} Make $u$ active
        \STATE \hspace{3em} \textbf{else if} $e$ is a bridge \textbf{then}
        \STATE \hspace{4em} Initiate the augmentation stage via $e$
        \STATE \hspace{3em} \textbf{end if}
        \STATE \hspace{2em} \textbf{end if}
        \STATE \hspace{1em} \textbf{end for}
        \STATE \hspace{1em} Make $v$ passive
        \STATE \textbf{end while}
        \STATE Terminate with a maximum $s$-$t$ flow and a minimum $s$-$t$ cut
    \end{algorithmic}
\end{algorithm}
\begin{algorithm}[H]
    \caption{BK stage: augmentation}
    \label{alg:bkaugment}
    \begin{algorithmic}
        \STATE \textbf{Input:} A bridge $e$
        \STATE Find the augmenting path induced by parent edges from $e$
        \STATE Find the bottleneck capacity $\Delta$ of the path
        \STATE Push $\Delta$ flow along the path and find orphans
        \STATE Initiate the adoption stage with the orphans
    \end{algorithmic}
\end{algorithm}
\begin{algorithm}[H]
    \caption{BK stage: adoption}
    \label{alg:bkadoption}
    \begin{algorithmic}
        \STATE \textbf{while} there are orphans in $S\cup T$ \textbf{do}
        \STATE \hspace{1em} Select an orphan $o\in S\cup T$
        \STATE \hspace{1em} Find a potential parent $u$ of $o$ such that\\\hspace{2em}\texttt{findroot}$(u)$ is a terminal
        \STATE \hspace{1em} \textbf{if} there is such a $u$ \textbf{then}
        \STATE \hspace{2em} Let $u$ be the parent of $o$ 
        \STATE \hspace{1em} \textbf{else}
        \STATE \hspace{2em} Make $o$ free 
        \STATE \hspace{2em} Make children of $o$ orphans
        \STATE \hspace{2em} Make potential parents of $o$ active
        \STATE \hspace{1em} \textbf{end if}
        \STATE \textbf{end while}
    \end{algorithmic}
\end{algorithm}

\begin{figure}[t]
    \centering
    \noindent\resizebox{0.43\textwidth}{!}{\def\thisscale{1.5}
\def\scolor{blue}
\def\tcolor{red}
\def\othercolor{black!35!white}

\begin{tikzpicture}[
        stdnode/.style={
            circle,
            draw=black,
            fill=gray,
            fill opacity = 0.3,
            text opacity=1,
            inner sep=0pt,
            minimum size=20pt,
            font=\large
        },
        myarrow/.style={
            -Stealth,
            font=\tiny,
            line width=1pt
        },
        node distance=0.6cm and 1.2cm,
    ]
    \useasboundingbox (-0.375,-0.375) rectangle (9.375, 3.375);

    \node[stdnode] (s) at (0, \thisscale) {$s$};
    \node[stdnode] (t) at (\thisscale*6,\thisscale) {$t$};

    \foreach \x/\label/\param in {1//,2//,3/$v_1$/,4//,5//}
        \node[stdnode,\param] (v\x_2) at (\thisscale*\x,\thisscale*2) {\label};
    \foreach \x/\label/\param in {1//,2/$v_2$/,3/$o_2$/,4//,5//}
        \node[stdnode,\param] (v\x_1) at (\thisscale*\x,\thisscale*1) {\label};
    \foreach \x/\label/\param in {1/$o_1$/,2//,3//,4/$u_1$/,5/$u_2$/}
        \node[stdnode,\param] (v\x_0) at (\thisscale*\x,\thisscale*0) {\label};

    \doubleflowedge{s}{v1_1}{}{}{\scolor}{\othercolor}{opacity=0}
    \shadebentedge{s}{v1_0}
    \doubleflowedge{s}{v1_0}{}{}{\othercolor, dotted}{\othercolor}{opacity=0}
    \doubleflowedge{s}{v1_2}{}{}{\othercolor, dotted}{\othercolor}{opacity=0}
    
    \doubleflowedge{v1_1}{v1_2}{}{}{\scolor}{\othercolor}{opacity=0}
    \doubleflowedge{v1_1}{v2_1}{}{}{\othercolor}{\othercolor}{opacity=0}
    \doubleflowedge{v1_0}{v1_1}{}{}{\othercolor}{\othercolor, dotted}{opacity=0}
    \doubleflowedge{v1_2}{v2_2}{}{}{\scolor}{\othercolor}{opacity=0}
    \shadebentedge{v1_0}{v2_0}
    \doubleflowedge{v1_0}{v2_0}{}{}{\scolor}{\othercolor}{opacity=0}
    \doubleflowedge{v2_0}{v2_1}{}{}{\scolor}{\othercolor}{opacity=0}
    \shadebentedge{v2_0}{v3_0}
    \doubleflowedge{v2_0}{v3_0}{}{}{\scolor}{\othercolor}{opacity=0}
    \doubleflowedge{v2_1}{v3_1}{}{}{\othercolor}{\othercolor}{opacity=0}
    \doubleflowedge{v2_1}{v2_2}{}{}{\othercolor}{\othercolor}{opacity=0}
    \doubleflowedge{v2_2}{v3_2}{}{}{\scolor}{\othercolor}{opacity=0}
    \doubleflowedge{v3_0}{v4_0}{}{}{\othercolor}{\othercolor}{opacity=0}
    \shadebentedge{v3_0}{v3_1}
    \doubleflowedge{v3_0}{v3_1}{}{}{\othercolor, dotted}{\othercolor}{opacity=0}
    \shadebentedge{v3_1}{v4_1}
    \doubleflowedge{v3_1}{v3_2}{}{}{\othercolor}{\othercolor}{opacity=0}
    \doubleflowedge{v3_1}{v4_1}{}{}{\scolor}{\othercolor}{opacity=0}
    \doubleflowedge{v3_2}{v4_2}{}{}{\othercolor, dotted}{\othercolor}{opacity=0}
    \doubleflowedge{v4_0}{v4_1}{}{}{\othercolor}{\othercolor}{opacity=0}
    \doubleflowedge{v4_0}{v5_0}{}{}{\othercolor}{\othercolor}{opacity=0}
    \doubleflowedge{v4_1}{v5_1}{$b_2$}{}{\othercolor}{\othercolor}{above=1mm, text=black!50!white, font=\large}
    \doubleflowedge{v4_1}{v5_1}{}{}{\othercolor}{\othercolor}{opacity=0}
    \doubleflowedge{v4_1}{v4_2}{$b_1$}{}{\othercolor, dotted}{\othercolor}{left=1mm, text=black!50!white, font=\large}
    \shadebentedge{v4_1}{v4_2}
    \doubleflowedge{v4_1}{v4_2}{}{}{\othercolor, dotted}{\othercolor}{opacity=0}
    \shadebentedge{v4_2}{v5_2}
    \doubleflowedge{v4_2}{v5_2}{}{}{\tcolor}{\othercolor}{opacity=0}
    \doubleflowedge{v5_0}{v5_1}{}{}{\othercolor}{\othercolor}{opacity=0}
    \doubleflowedge{v5_1}{v5_2}{}{}{\tcolor}{\othercolor}{opacity=0}
    \doubleflowedge{v5_2}{t}{}{}{\tcolor}{\othercolor}{opacity=0}
    \shadebentedge{v5_2}{t}
    \doubleflowedge{v5_1}{t}{}{}{\othercolor, dotted}{\othercolor}{opacity=0}
    \doubleflowedge{v5_0}{t}{}{}{\othercolor}{\othercolor, dotted}{opacity=0}

    
\end{tikzpicture}} %
    \caption{Example state of the BK algorithm after the augmentation stage. Edges of the out-forest $S$ are blue and edges of the in-forest $T$ are red. Edges in the augmenting path used in the recent augmentation stage are highlighted in green. The vertices $u_1, u_2\in V$ are free vertices. The edges $b_1, b_2\in E_f$ are bridges. Saturated residual edges are dotted. The push saturates the bridge $b_1$ and the highlighted edges to $o_1$ and $o_2$ (causing $o_1$ and $o_2$ to be orphans). The orphan $o_2$ has two potential parents $v_1,v_2\in S$. Of these, only $v_1$ is connected to the terminal and can be used to reconnect $o_2$ to $s$.}
    \label{fig:exBK}
\end{figure}
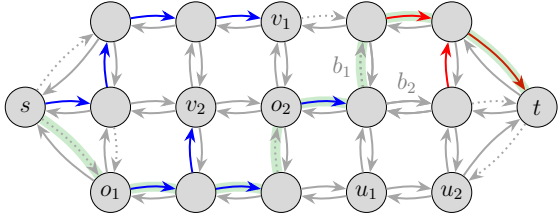

\subsubsection{Heuristics}

In the BK algorithm, each vertex stores its last known distance from its terminal root and a timestamp. However, if a vertex has its distance and timestamp updated, the algorithm does not traverse further down the tree to propagate this information. As a result, these values are only heuristic. The BK algorithm uses these heuristic values in the adoption stage. We say that a vertex is \emph{marked} if it stores an updated timestamp, which it does if it is confirmed to be connected to a terminal during this adoption stage. The \texttt{findroot} procedure will terminate early if it encounters a marked vertex. Additionally, the algorithm considers all potential parents of an orphan that are connected to a terminal and selects the one with the smallest distance. These heuristic values are also used in the growth stage. If the non-saturated residual edge $e$ has both endpoints in the same tree, the edge can replace an existing parent edge if this shortcuts the tree.


\subsubsection{Implementation}

We denote by \texttt{BK} the implementation (version $3.04$) of the Boykov-Kolmogorov algorithm, the most recent version from \cite{boykov2004a}, available at \url{https://pub.ista.ac.at/~vnk/software.html}. We denote by \texttt{mBK} and \texttt{mBK-r} the implementations by \cite{jensen2023a}, available at \url{https://github.com/patmjen/maxflow_algorithms}.

\texttt{BK} uses pointers to reference vertex and internal residual edge objects. If the datasets are small enough, \texttt{BK} can be compiled with 32-bit pointers instead of 64-bit pointers, which are standard on most modern machines. \texttt{BK} stores the residual edges in the order they are presented to it. Each residual edge additionally stores three pointers: to the head of the residual edge, to the mirror residual edge, and to the next outgoing residual edge from this vertex. \texttt{BK} uses the last of these pointers to iterate outgoing residual edges as a linked list. \texttt{BK} processes the orphans in a queue and enqueues those closest to the terminal first.

Both \texttt{mBK} and \texttt{mBK-r} store vertices and internal residual edges in lists and refer to them using their indices in those lists (instead of pointers). The \texttt{mBK} implementation does not store the index to the mirror in a residual edge, as the mirror is stored adjacent in memory to the residual edge. On the other hand, \texttt{mBK-r} does not store the index to the next residual edge as it uses a temporary data structure to pack outgoing residual edges contiguously in memory. These implementations save the most space over \texttt{BK} for datasets that exceed the address space of 32-bit compiled programs, yet contain sufficiently few elements for 32-bit indices to suffice. Both \texttt{mBK} and \texttt{mBK-r} also use a queue to process the orphans, but have an additional queue for orphans produced by orphans and empties this queue first.

All three implementations have each vertex of the residual graph store the next active vertex, the first outgoing edge, the parent vertex, the two heuristic values (timestamp and distance), and the corresponding terminal edges. When building the residual graph, they process all terminal edges in the following way. If a vertex has multiple terminal edges from $s$, they sum their capacities (and similarly for $t$). If a vertex is incident with both $s$ and $t$, they push flow along this augmenting path. Thus, after building the residual graph, $s$ and $t$ are passive, each vertex has at most one terminal residual edge (from $s$ or to $t$), and the linked list of active vertices contains exactly the vertices with a terminal residual edge. The number of bytes required to represent the residual graph in each implementation is shown in Table~\ref{tab:spacetable}.

\section{Running Time of the Boykov-Kolmogorov Algorithm} \label{sec:bk_runtime}

The authors of BK present an upper bound on the running time of $O(mn^2|C|)$ \cite{boykov2004a}. We show a simple analysis that improves this to $O(mn|C|)$ for general graph instances or $O(m^2n)$ for many instances in computer vision, where $n$ is the number of vertices, $m$ is the number of edges, and $|C|$ is the capacity of the minimum $s$-$t$ cut in the graph. As in previous work, we assume integer capacities.

Since each augmentation pushes positive flow, the algorithm loops through the three stages at most $|C|$ times. We partition the work during one such loop into the three stages. The growth stage only processes active vertices. As active vertices are made passive after processing, each vertex is processed at most once during a single growth stage. Processing a vertex can explore its (outgoing or incoming) residual edges, each of which takes $O(1)$ work and possibly ends the stage. Therefore, a single growth stage will spend $O(1)$ work per residual edge and terminate after $O(m)$ steps. The augmentation stage performs work proportional to the length of the augmenting path from $s$ to $t$, and thus terminates after $O(n)$ steps. In the adoption stage, each vertex is processed as an orphan at most once. To see this, note that any new orphans created in this stage come from orphan-rooted trees. But each vertex that has been processed as an orphan is either reconnected to the terminal or made free. When processing an orphan, we can explore all of its (outgoing or incoming) residual edges twice (once to look for a potential parent and once if the orphan is made free) for a total of $O(m)$ explored residual edges. Each such exploration might initiate one \texttt{findroot} search, but takes $O(1)$ work otherwise. Each \texttt{findroot} search traverses parent edges, so it terminates after $O(n)$ steps. Combining these bounds, each adoption stage will terminate after $O(mn)$ time. Also note that the order of processing orphans does not change this bound. In total, the BK algorithm terminates after $O(mn|C|)$ steps. The bottleneck of the analysis is the $O(mn)$ steps of \texttt{findroot} in each adoption stage. This analysis improves on the $O(mn^2|C|)$ running time bound in \cite{boykov2004a}.

We now consider how the complexity might be affected by assumptions on the graph. In particular, we use the graph assumptions to bound the number of augmenting paths (and thus loop iterations). We focus on one assumption that is present in many computer vision instances, namely that of constant edge capacities (e.g.,\ $100$), either for internal edges or for terminal edges. If the capacity of each terminal edge is constant, we get a cut of size $O(m)$ by cutting all $O(m)$ terminal edges. But then the minimum cut is of capacity $|C|=O(m)$. On the other hand, if the capacity of each internal edge is constant, we note that each vertex has at most one terminal residual edge after the $O(m)$ processing of terminal edges. At this point, the residual graph $G_f$ also corresponds to a graph $G'$ where each vertex only has at most one terminal edge. As all augmenting paths in $G_f$ must use internal residual edges, $G'$ will have a cut $C'$ of capacity $|C'|=O(m)$. In both cases, the total flow that can be pushed after building the graph is $O(m)$. Therefore, a constant edge capacity bound (on either terminal or internal edges) implies that BK finds at most $O(m)$ augmenting paths, resulting in a BK algorithm running time of $O(m^2n)$. Many graph instances from computer vision also have constant vertex degrees for non-terminal vertices. After the initial $O(m)$ processing of all terminal edges, we are left with a residual graph containing $n$ vertices, $O(n)$ terminal residual edges, and $m_i=O(n)$ internal residual edges. Plugging this into the analysis gives a $O(n^3)$ running time bound.

\section{Compact Graph Representation}\label{sec:compact_graph}

We present our compact graph representation for the BK algorithm. Given a graph instance $G$, we maintain the flow $f$ by representing the corresponding residual graph $G_f$. The flow $f$ is initially empty. We additionally maintain data structures specific to the BK algorithm. These are: 1) the out-forest $S$ and the in-forest $T$; 2) the linked-list of active vertices; and 3) the distance and timestamp heuristics for all vertices. We represent the residual graph and the BK specific data using $8$ bytes for each internal edge and $16$ bytes for each vertex (including its incident terminal edges). Similar to \texttt{BK}, \texttt{mBK} and \texttt{mBK-r}, we represent only the outgoing residual edges from $s$ and the incoming residual edges to $t$. The distribution of bytes to values and how it compares to existing approaches is found in Table~\ref{tab:spacetable}. Many of the techniques presented here are applicable to other maximum $s$-$t$ flow algorithms, particularly ones using augmenting paths.

\subsection{Data Structure}

\begin{figure}[!t]
\centering
\noindent\resizebox{0.48\textwidth}{!}{\begin{tikzpicture}[
  x=0.57cm,
  y=1cm,
  line cap=round,
  line join=round,
  every node/.style={font=\small},
  index/.style={font=\footnotesize, text=black!60, anchor=south},
  content/.style={font=\normalsize},
  direction/.style={-{Stealth[length=4pt,width=5pt]}, black!55,
                    line width=0.55pt},
  boundary/.style={black!72, line width=0.65pt},
  wordtick/.style={black!32, line width=0.4pt}
]

\def\slotheight{0.90}
\def\firsty{0}
\def\secondy{-1.75}
\def\indexdist{0.07}
\def\indexxshift{0.5}

\newcommand{\memoryobject}[6]{%
  \path[fill=#5] (#1,#2) rectangle ({#1+#3},{#2+\slotheight});
  \draw[boundary] (#1,#2) rectangle ({#1+#3},{#2+\slotheight});
  \node[content] at ({#1+#3/2},{#2+\slotheight/2}) {#4};
  \node[index] at ({#1+\indexxshift},{#2+\slotheight+\indexdist}) {$#6$};
  \pgfmathtruncatemacro{\nticks}{#3-1}%
  \ifnum\nticks>0
    \foreach \k in {1,...,\nticks}{%
      \draw[wordtick] ({#1+\k},{#2}) -- ({#1+\k},{#2+0.11});
    }%
  \fi
}

\useasboundingbox (-1,\secondy) rectangle (19,1.25);

\memoryobject{0}{\firsty}{4}{$a$}{black!10}{0}
\memoryobject{4}{\firsty}{2}{$(a,b)$}{white}{4}
\memoryobject{6}{\firsty}{2}{$(a,c)$}{white}{6}
\memoryobject{8}{\firsty}{4}{$b$}{black!10}{8}
\memoryobject{12}{\firsty}{2}{$(b,c)$}{white}{12}
\memoryobject{14}{\firsty}{2}{$(b,a)$}{white}{14}
\memoryobject{16}{\firsty}{2}{$(b,d)$}{white}{16}

\draw[direction] (18.12,{\firsty+\slotheight/2}) -- (18.85,{\firsty+\slotheight/2});
\draw[direction] ({-1+0.12},{\secondy+\slotheight/2}) -- ({-1+0.85},{\secondy+\slotheight/2});

\memoryobject{0}{\secondy}{4}{$c$}{black!10}{18}
\memoryobject{4}{\secondy}{2}{$(c,b)$}{white}{22}
\memoryobject{6}{\secondy}{2}{$(c,d)$}{white}{24}
\memoryobject{8}{\secondy}{2}{$(c,a)$}{white}{26}
\memoryobject{10}{\secondy}{4}{$d$}{black!10}{28}
\memoryobject{14}{\secondy}{2}{$(d,c)$}{white}{32}
\memoryobject{16}{\secondy}{2}{$(d,b)$}{white}{34}

\end{tikzpicture}}
\caption{Example layout of the interleaved list in our compact graph representation corresponding to the residual graph $G_f$ in \figurename~\ref{fig:exGf}. A vertex takes up 4 integers and a residual edge takes up 2 integers. The information saved in each location is detailed in Table~\ref{tab:spacetable}.}
\label{tab:memorylayout}
\end{figure}

\subsubsection{Memory Layout of the Graph}
We pack the residual graph in memory such that vertices and residual edges are interleaved. Specifically, we have one list that alternates between vertices and their corresponding outgoing residual edges. This interleaving is visualized in \figurename~\ref{tab:memorylayout}. The benefits of this interleaving are threefold: locating the outgoing residual edges of a vertex does not require an explicitly stored reference, accessing the outgoing residual edges is sequential in memory, and iterating over the outgoing residual edges is fast as they are contiguous in memory.

\subsubsection{Graph Edges to Residual Edges}
We remove any edge in $G$ that has a capacity of $0$. Afterwards, we \emph{merge parallel residual edges}, where two residual edges are called \emph{parallel} if they have the same source and target. Specifically, we use a single pair of residual edges $(u,v), (v,u)\in E_f$ to represent all edges between $u$ and $v$ in the original graph $G$. With an empty flow $f$, the residual capacity of $(u,v)\in E_f$ is the sum of the capacities of all edges from $u$ to $v$ in the original edge set $E$, and similarly for $(v,u)\in E_f$. A saturated residual edge is still in our representation, but can never be part of $S$ or $T$. If a pair of residual edges from $E_f$ corresponds to multiple edges in $E$, we do not specify how to distribute the residual capacity from $G_f$ to the corresponding edges in $G$. Our resulting graph representation has no parallel residual edges. Even if no parallel edges exist in $G$, it is possible to have parallel edges in $G_f$. For instance, this happens when each edge from $E$ has a unique pair of residual edges and if there is a pair $u,v\in G_f$ where $(u,v),(v,u)\in E$. Note that we do not change the minimum $s$-$t$ cut (or the maximum flow value) by removing $0$-capacity edges from $G$ or by merging parallel residual edges in $G_f$ because the sum of the capacity from any $u\in V$ to any $v\in V$ is preserved.

\subsubsection{Referencing a Vertex}

We represent a reference to a vertex using its index in the interleaved list. If $G$ is so large that we are unable to index the interleaved list using a 32-bit index, we say that the dataset is \emph{extra-large}. For extra-large datasets, we instead let each vertex reference stored in an edge be the index offset between the incident vertices. We say that we use the \emph{relative index}. We get the (absolute) index of the head of an edge by adding the index of the current vertex. For extra-large datasets, we also use the relative index for accessing the next active vertex. As the reference to the next active vertex does not have to follow the edges of the graph, there are cases (for extra-large datasets) where a few of the values are too large for the data type in the vertices. In these cases, we note this in the vertex and store the too-large value in an external queue using the (absolute) index. The external queue is a subsequence of the linked-list containing active vertices. Therefore, if we are looking for the index of the next active vertex from a vertex with a too-large next active relative index, the wanted index will be at the front of the external queue. Relative indices of $32$ bits suffice for most extra-large computer vision datasets, as internal edges of the graph often connect only the local neighborhood of a vertex, which has similarly labeled vertices. 

\subsubsection{Representing \texorpdfstring{$S$ and $T$}{S and T}}

In each vertex $v$, we use flags to indicate if that vertex is free, an orphan, or a tree vertex. If the vertex is a tree vertex, we also indicate if its parent edge is incident with a terminal, and whether $v\in S$ or $v\in T$. For each $v$ that is a tree vertex with a parent edge that is not incident with a terminal, we represent the parent of $v$ implicitly by the head of its first outgoing residual edge. We maintain this by swapping the order of outgoing residual edges from a vertex when it is connected (or reconnected) to the trees rooted at $s$ or $t$. We can do this, as we do not change which edges the algorithm can explore (only the order in which it does so).

\subsubsection{Representing the Mirror}
Throughout the algorithm, we often need to find the mirror $e'\in E_f$ of a residual edge $e\in E_f$ from a vertex $v$. We find this mirror using one of two approaches. The \emph{index-based} approach uses the index of $e'$ in the list of outgoing residual edges from the head of $e$. We get the index of $e'$ in the interleaved list by adding the saved index to the index of the head of $e$. Alternatively, the \emph{scan-based} approach goes over all outgoing residual edges from the head of $e$ to find a residual edge with $v$ as its head. As no parallel edges are present, this residual edge is $e'$. We use the index-based approach for all internal edges $e\in E_f$ that go from a tree vertex $v$ to the parent of $v$, and the scan-based approach otherwise (when reconnecting orphans or growing $S$ and $T$). The index for the index-based approach is never larger than the vertex degree of the non-terminal parent and therefore encodable using few bits. For the scan-based approach, scanning for an edge is fast as the outgoing residual edges are contiguous in memory. As in prior work (e.g.,\ \cite{jensen2023a, goldberg2011a}), each residual edge stores a flag indicating whether its mirror has any residual capacity left. This can avoid searching for (and accessing) the mirror residual edge if we only need to check whether it is saturated (when checking $e$ during growth and looking for potential parents in the adoption stage). 

\subsubsection{Overloading Values in a Vertex}
We store six values in each vertex $v$: 1) the next active vertex from $v$; 2) its timestamp; 3) its distance; 4) its degree; 5) if $v\in S$ or if $v\in T$; and 6) the index of the edge from the parent of $v$ (finding the mirror using the index-based approach). The latter (6) is additionally used to indicate if $v$ is free, an orphan, or connected directly to the terminal. Only if none of these states apply does the latter value (6) make sense (there is a residual edge from a non-terminal parent to $v$) and is used. We also overload the timestamp value (2) to represent the terminal residual capacity to or from $v$. We can do this, as the timestamp is not needed if the vertex is connected directly to the terminal.


\subsection{Preprocessing}
We construct our compact representation of the residual graph $G_f$ using three passes over the graph $G$. In the first two passes we get the number of internal residual edges after merging (degree) for each vertex, and in the final pass we insert the graph information into the correct location. Specifically, in the first pass of $G$, we count the number of residual edges from each vertex, including potential parallel residual edges. For each vertex, we then use its count to allocate a contiguous array of that length. In the second pass of $G$, we use these arrays to keep track of which outgoing residual edges we have already seen from each vertex and only append the unseen residual edges (those that do not create parallel residual edges). The array stores only the relative index (32-bit) of the seen outgoing residual edges. In both of the first two passes, we ignore edges in $G$ with capacity $0$ and the terminal edges. Having the degrees, we then know where to place each vertex in the interleaved list such that it has space for its outgoing residual edges. The third pass inserts the values from $G$ into the correct location of our representation. If we need to add multiple terminal residual edges to the same vertex in this third pass, we directly merge them (or push flow along them) such that only one terminal edge remains. When adding an internal residual edge in this third pass, we either insert a new residual edge into the interleaved list or add its capacity to an already inserted (parallel) residual edge. Any time we check which residual edges we have seen before during preprocessing, we iterate the current list of seen residual edges, which is fast to do as the lists are contiguous in memory.

\subsection{Assumptions}
Any representation assumes that the values can fit within the bits allocated to the values. Our representation additionally requires that: 1) the residual capacity of each edge in $G_f$ always fits in a $31$-bit unsigned integer; 2) each relative index (index offset in the interleaved list between neighboring vertices) fits in a $32$-bit signed integer; and 3) the degree of each non-terminal vertex is bounded. The latter assumption (3) has a hard constraint that degrees should fit in the allocated bits ($deg(v)<8.3M$). However, we also use it to ensure that it is fast for the algorithm to iterate through, and swap the order of, outgoing residual edges. The first assumption (1) will be satisfied when each pair of vertices has the sum of their edge capacities less than $2.1B$. The second assumption (2) can be avoided in most cases (even for extra-large datasets) by choosing a smart labeling of the vertices, as done in Section~\ref{sec:largevolumes}. These assumptions are often realized for graphs coming from vision datasets, and can be validated during preprocessing.

\begin{table*}[!t]
\centering
\begin{threeparttable}
\caption{Number of bytes used to represent a residual graph, and what the bits are allocated to, for the different BK algorithm implementations. This does not include the space for temporary data structures or possible savings from merging parallel residual edges.} 
\label{tab:spacetable}
\begin{tabular}{llcccc} 
\toprule
 & Description & \texttt{BK}~\cite{boykov2004a} & \texttt{mBK}~\cite{jensen2023a} & \texttt{mBK-r}~\cite{jensen2023a} & \texttt{fcBK}\,/\,\texttt{cBK}~(ours) \\ \cmidrule(r){1-2} \cmidrule(rl){3-3}\cmidrule(rl){4-4} \cmidrule(rl){5-5} \cmidrule(l){6-6}
\multicolumn{1}{l}{\multirow{9}{*}{\rotatebox[origin=c]{90}{\shortstack{Bits\,/\\Vertex}}}} & Next active               & $|\texttt{pointer}|$  & $|\texttt{index}|$& $|\texttt{index}|$& $32$         \\ 
& Timestamp                 & $|\texttt{long}|$     & $32$              & $32$              & $32$       \\ 
& Distance                  & $|\texttt{int}|$      & $16$              & $16$              & $17$      \\ 
& Vertex degree             & -                   & -               & -               & $23^*$          \\ 
& Index of the edge from the parent & -                   & -               & -               & $23$         \\ 
& Is the vertex in $S$ or $T$?         & $1$                   & $1$               & $1$               & $1$          \\
& Parent vertex             & $|\texttt{pointer}|$  & $|\texttt{index}|$& $|\texttt{index}|$& -          \\
& Terminal residual capacity& $32^\dagger$     & $32^\dagger$ & $32^\dagger$ & -      \\ 
& First outgoing edge       & $|\texttt{pointer}|$  & $|\texttt{index}|$& $|\texttt{index}|$& -         \\ \midrule 
\multicolumn{1}{l}{\multirow{5}{*}{\rotatebox[origin=c]{90}{\shortstack{Bits\,/\\Res. edge}}}} & Head of edge              & $|\texttt{pointer}|$  & $|\texttt{index}|$& $|\texttt{index}|$& $32$        \\ 
& Mirror edge               & $|\texttt{pointer}|$  & -               & -               & -          \\
& Next outgoing edge        & $|\texttt{pointer}|$  & -               & -               & -          \\
& Is the mirror edge saturated?   & -                   & -               & $1$               & $1$          \\
& Residual capacity    & $32^\dagger$     & $32^\dagger$ & $32^\dagger$ & $31^*$         \\ \midrule 
\multicolumn{1}{l}{\multirow{3}{*}{\rotatebox[origin=c]{90}{\shortstack{Bytes\,/\\$G_f\,^{**}$}}}}& Small dataset$^\ddagger$    & $28n+16m_{i}$             & $23n+12m_{i}$         & $23n+13m_{i}$         & $16n+8m_{i}$     \\ 
& Large dataset$^\ddagger$    & $48n+32m_{i}$             & $23n+12m_{i}$         & $23n+13m_{i}$         & $16n+8m_{i}$     \\
& Extra-large dataset$^\ddagger$ & $48n+32m_{i}$             & $35n+20m_{i}$         & $35n+21m_{i}$         & $16n+8m_{i}$     \\ \bottomrule
\end{tabular}
\begin{tablenotes}
    \item [*] We can represent degrees up to $8.3M$ and residual capacities up to $2.1B$. 
    \item [**] Size on our machine (\texttt{LP64} data model) using \texttt{sizeof}, which includes buffer space to align values inside an object. The size is given in bytes (8 bits) for a residual graph with $n$ vertices and $m_i$ internal residual edges. 
    \item [$\dagger$] Can be adjusted to the dataset, but doing so might change the space requirements.
    \item [$\ddagger$] We say that a dataset is \emph{small} if we can build the corresponding residual graph and find its minimum $s$-$t$ cut using a 32-bit compiled program. A \emph{large} dataset is a dataset that is not small (requires a 64-bit compiled program), but where 32-bit indices suffice for indexing vertices and edges in the residual graph. An \emph{extra-large} dataset is a dataset that is neither small nor large.
\end{tablenotes}
\end{threeparttable}
\end{table*}

\section{The Fast Boykov-Kolmogorov Algorithm}\label{sec:fcbk}

The \emph{fast Boykov-Kolmogorov} (fBK) algorithm is our new algorithm. It is a variant of the BK algorithm \cite{boykov2004a}. It has algorithmic modifications that improve the worst-case behavior. Like the BK algorithm, our algorithm finds augmenting paths by looping through the three stages called \emph{growth}, \emph{augmentation}, and \emph{adoption} (see Section~\ref{sec:bk} for details). Algorithmically, fBK differs from BK only by introducing orphan-path flags that eliminate duplicate work in \texttt{findroot}, which is the bottleneck of the running time analysis as identified in Section~\ref{sec:bk_runtime}. 

\subsection{Orphan-Path Flags}
An \emph{orphan-path flag} is a flag that can be set for tree vertices. A flagged vertex is confirmed to be in a tree rooted at an orphan within the current adoption stage. The \texttt{findroot} procedure will terminate early if it encounters a flagged vertex. This mirrors how existing heuristics mark vertices connected to terminals. Specifically, we alter \texttt{findroot} if the search encounters an orphan or a vertex with an orphan-path flag. In this case, the algorithm sets orphan-path flags in all vertices on the path searched. Additionally, we change what happens when processing an orphan with an orphan-path flag. If that orphan is made free, the algorithm removes its orphan-path flag. On the other hand, if the orphan is reconnected to the terminal-rooted tree, the algorithm traverses all flagged vertices in the tree rooted at that orphan using a depth-first traversal, and removes their flags. When removing the flag of a vertex where its tree was just reconnected to the terminal, the algorithm additionally marks the vertex as connected to the terminal. 

Intuitively, we would like to use the flagged and marked vertices to avoid repeated work. To see why this is possible, we need the following two invariants.

\begin{lemma}\label{lem:opf}
    A vertex with an orphan-path flag is either in a tree rooted at an orphan, or will have its orphan-path flag removed before the next \texttt{findroot} call.
\end{lemma}

\ifthenelse{\useieeetemplate=1}{\begin{IEEEproof}}{\begin{proof}}
    Let us assume that no orphan-path flags are present at the start of an adoption stage. Consider any vertex $v$ with an orphan-path flag. By construction, the algorithm has visited $v$ with a \texttt{findroot} search that found an orphan or an orphan-path flag during this stage. As flags are removed for vertices made free during this stage, $v$ is a tree vertex. We are done if $v$ is an orphan and since $v$ cannot be terminal, we can assume that $v$ has a parent that is a tree vertex. If the parent of $v$ is different from the time of the \texttt{findroot} search that flagged $v$, this adoption stage must have connected $v$ with a potential parent and removed its orphan-path flag immediately. Therefore, the parent of $v$ has also been visited by a \texttt{findroot} search during this stage and had an orphan-path flag set too. If the flag of the parent is still set, we continue recursively with the parent of $v$. The recursion will stop when we find an orphan or a parent where the flag is no longer set. In the latter case, either that parent or its orphan has been reconnected to the terminal-rooted tree. But then we are currently in the process of removing all orphan-path flags in the tree, ensuring that $v$ has its flag removed before the next \texttt{findroot} call. As no orphans remain at the end of the adoption stage, all orphan-path flags will be removed at the end of the adoption stage.
\ifthenelse{\useieeetemplate=1}{\end{IEEEproof}}{\end{proof}}

\begin{lemma}\label{lem:marked_vertices}
    A vertex that is marked during an adoption stage is connected to the terminal. 
\end{lemma}

\ifthenelse{\useieeetemplate=1}{\begin{IEEEproof}}{\begin{proof}}
    Let $v$ be a vertex that has been marked during an adoption stage. Then $v$ was marked in one of three cases. If $v$ was marked due to a \texttt{findroot} search that encountered a terminal, $v$ is connected to the terminal, so we are done. If $v$ was marked due to a \texttt{findroot} search that encountered another marked vertex, we continue recursively from that vertex. Finally, if $v$ was marked when its orphan-path flag was removed, it must be connected to the vertex $u$ that reconnected the orphan of $v$. But then $u$ was previously a potential parent from which a \texttt{findroot} search left it marked, so we continue recursively from $u$. In all cases, $v$ is connected to the terminal.
\ifthenelse{\useieeetemplate=1}{\end{IEEEproof}}{\end{proof}}

Combining these two gives the following.

\begin{corollary}\label{cor:findroot_correct}
    A \texttt{findroot} search in fBK correctly identifies if a potential parent is connected to the terminal.
\end{corollary}

\ifthenelse{\useieeetemplate=1}{\begin{IEEEproof}}{\begin{proof}}
    A \texttt{findroot} search terminates when it encounters either the terminal, an orphan, an orphan-path flag, or a marked vertex. We can use Lemma~\ref{lem:opf} and Lemma~\ref{lem:marked_vertices} to see that we correctly report if the potential parent is connected to the terminal in all cases.
\ifthenelse{\useieeetemplate=1}{\end{IEEEproof}}{\end{proof}}

We can now prove correctness.

\begin{lemma}\label{lem:fcbk_correctness}
    The fBK algorithm terminates with a maximum flow.
\end{lemma}
\ifthenelse{\useieeetemplate=1}{\begin{IEEEproof}}{\begin{proof}} 
    The BK algorithm terminates with a maximum flow~\cite{kolmogorov2004b}. We now compare BK and fBK without the growth-stage heuristic that can shortcut the tree. In this setting, we show that BK and fBK have identical forests ($S$ and $T$) after each stage. We can assume that BK and fBK have identical forests up to some adoption stage. As BK and fBK differ only in \texttt{findroot}, we can use Corollary~\ref{cor:findroot_correct} to see that all orphans are processed, and thus reconnected, identically. But then BK and fBK will have identical forests after the adoption stage. The two algorithms might have different heuristic values in some vertices, but this does not change the behavior of the growth stage (without the shortcut heuristic) or the augmentation stage. Hence BK and fBK will have identical forests for the next adoption stage too. As we can continue these arguments, the forests will also be identical when BK terminates. Finally, we note that short-cutting during the growth-stage is heuristic and does not affect the flow~\cite{boykov2004a}, ensuring that fBK terminates with a maximum flow.
\ifthenelse{\useieeetemplate=1}{\end{IEEEproof}}{\end{proof}}

\subsection{Time Complexity of the fBK algorithm}

The description of the running time of the BK algorithm in Section~\ref{sec:bk_runtime} also applies to the fBK algorithm. However, we have made changes to the \texttt{findroot} procedure and the processing of orphans, which affect the time complexity of the adoption stage for the fBK algorithm. We can split the adoption stage complexity into four major components. These are: 1) exploring edges to look for a potential parent; 2) searching along edges in \texttt{findroot}; 3) removing orphan-path flags; and 4) exploring edges if the orphan was not reconnected. As described in Section~\ref{sec:bk_runtime}, the first and last components (1 and 4) do $O(1)$ work per edge for a total of $O(m)$ work, even if we include the work of initializing \texttt{findroot} from a potential parent and (possibly) reconnecting the orphan.

For the second component (2), we note that whenever \texttt{findroot} traverses a parent edge during a search, the algorithm either marks the vertex as connected to a terminal or flags it as connected to an orphan. In the latter case, the algorithm only removes the flag if it makes the vertex free or if it marks the vertex as connected to a terminal. Both cases imply that within a stage, \texttt{findroot} will only continue the search from a vertex (using a parent edge) once for a total of $O(n)$ edges traversed. For the third component (3), we do not reconsider a vertex made free in this stage, so the remaining work follows from reconnected orphans. Here we note that each vertex traversed with the purpose of removing its flag is marked as connected to the terminal, so it will not have its orphan-path flag set later in this stage. Therefore, each vertex is only traversed once in a single stage with the purpose of removing its flag, resulting in $O(m)$ work removing orphan-path flags. 

Combining these components, a single adoption stage will finish after $O(m)$ work. The result is a total running time of $O(m|C|)$ for the fBK algorithm. If we additionally have a constant upper bound on the internal or terminal edge capacities (as in Section~\ref{sec:bk_runtime}), this further reduces to $O(m^2)$ (and $O(n^2)$ with constant vertex degrees for non-terminal vertices).

\subsection{Implementation}

We let \texttt{fcBK} denote our implementation of the fBK algorithm using the compact graph representation presented in Section~\ref{sec:compact_graph} and corresponding preprocessing. Our implementation further exploits that the timestamp is not needed if the vertex is part of an orphan-rooted tree. Instead, it overloads the timestamp to indicate the orphan-path flag. In total, \texttt{fcBK} requires only $16$ bytes per vertex and $8$ bytes per internal edge. The full allocation of bits to values and a comparison to the other implementations are found in Table~\ref{tab:spacetable}. As done in \cite{jensen2023a}, \texttt{fcBK} processes the orphans that come from a processed orphan right away and (only then) returns to the queue of orphans. 

\section{Benchmark Experiments} \label{sec:benchmark}

We compare variants and implementations of the BK algorithm through three experiments on a comprehensive set of computer vision benchmark datasets. The first experiment compares all implementations of the BK algorithm on time and graph memory required to find the maximum $s$-$t$ flow. The second experiment measures the peak memory usage of our implementations, which is the key metric for scalability to larger graph instances. The third experiment is an ablation that isolates the role of orphan-path flags in practice.


\subsection{Datasets}
We experiment on a large range of computer vision benchmark datasets from \url{doi.org/10.11583/DTU.17091101}. They include the widely used dataset from the University of Waterloo \cite{waterloo} and datasets from the surveys of \cite{verma2012a} and \cite{jensen2023a}. The datasets consist of graph instances coming from 3D voxel segmentation, stereo matching, multi-view reconstruction, super-resolution, texture restoration, cleaning neural network predictions, quadratic pseudo-boolean optimization (QPBO), surface fitting, and more \cite{jensen2023a}. All datasets have integer capacities, and many additionally have an upper bound of $100$ for the edge capacity of either the internal or the terminal edges. 

We benchmark $7$ datasets with multiple graphs (subproblems) and $54$ single-graph datasets for a total of $1719$ graph cuts. According to the size classification defined in Section~\ref{sec:compact_graph}, all datasets with multiple subproblems happen to be small, whereas the single-graph datasets comprise $28$ small datasets, $25$ large datasets, and one extra-large dataset. We only have a single extra-large dataset that fits in the memory of our machine for the experiments. This single example is insufficient to draw general conclusions on the implementation's running time for extra-large datasets. However, we do find it indicative of their space usage for graphs of this magnitude. 

\subsection{Implementations}

We compare our \texttt{fcBK} implementation with \texttt{BK}~\cite{boykov2004a}, \texttt{mBK}~\cite{jensen2023a}, and~\texttt{mBK-r} \cite{jensen2023a}. We also include \texttt{cBK}, which is our implementation of the BK algorithm. It is identical to \texttt{fcBK} except that it does not use orphan-path flags. For small datasets, we compile \texttt{BK} into a 32-bit compiled program and the other implementations into 64-bit compiled programs using 32-bit indices. For large datasets, we compile all implementations into a 64-bit compiled program and let \texttt{mBK}, \texttt{mBK-r}, \texttt{cBK}, and \texttt{fcBK} use 32-bit indices. For the extra-large dataset, our implementations use 32-bit relative indices while the other implementations use 64-bit indices/pointers. These compile options (including different options for small datasets) are chosen as they achieve the smallest time and memory footprint for each implementation. 

\subsection{Test Setup}
We run the experiments on a machine with an \texttt{x86\_64} AMD Ryzen 9 7950X3D processor at $4.2$~GHz with $128$~MB cache and $128$~GB of memory. Only the operating system and the experiments run on the machine during testing. All implementations are compiled using \texttt{gcc} version \texttt{11.4.0} with the following compile flags: \texttt{-O3 -march=native -mtune=native -DNDEBUG}. For the small datasets, we additionally compile \texttt{BK} using \texttt{-m32}. Every reported time (similarly for measured peak memory) value is the median over five runs. 

\subsection{Comparing BK Implementations}\label{exp:comp}
We run each implementation on every dataset and measure the time spent. We ensure that all implementations start with the same permutation of their initial queue of active vertices, as we find this to have a large impact on the running time for some datasets. The per-dataset \emph{solve} time and \emph{total} time appear in Table~\ref{tab:running_time_real_world}. The solve time includes only the time to find the maximum $s$-$t$ flow (and not extracting the cut). The total time includes the solve time as well as the preprocessing time (building temporary data structures, passing over the graph multiple times, loading values into the representation, ...). The total time does not include the time to load the graph into memory. We sum the times for problems that consist of multiple subproblems. The table also reports the number of vertices ($n$), the number of internal residual edges ($m_i^+$), and the number of internal residual edges after merging parallel residual edges ($m_i$). For datasets with multiple subproblems, we report the average sizes over the subproblems. Only our implementations merge parallel internal residual edges. We do not include the number of terminal edges, as these do not take up additional space (they are represented in the vertices).

Each solve time along with the corresponding \emph{graph memory} is averaged in \figurename~\ref{fig:combinedfig_bar} and plotted in \figurename~\ref{fig:combinedfig_rel}, both split into small and large datasets. The graph memory is the size of the allocated memory used to represent the residual graph. The graph memory does not consider the space for temporary data structures or any additional queues maintained (all implementations have orphan queues and for extra-large datasets our implementations have an external queue for too-large next active relative indices). See the following peak memory experiments for metrics that include these.

\subsection{Peak Memory Experiments}\label{exp:peak_mem}
We measure the peak memory of our implementations on a subset of the datasets, including the two large examples from Section~\ref{sec:largevolumes}. To mirror the practical setting, these experiments do not assume that the dataset is loaded into memory prior to measuring time and memory. The \emph{preprocessing memory} and \emph{solve memory} are compared to the \emph{graph memory} (see the comparison experiments) in Table~\ref{tab:scalability}. The preprocessing memory is the maximum resident set size (peak RSS) of a program that builds the residual graph from a file. The solve memory is the peak RSS of a full run (preprocessing and solve), where we reset the peak RSS to the current RSS after building the graph. We find the RSS of a full run (preprocessing and solve) to be identical to the preprocessing memory for the considered datasets. To ensure that the size of the dataset does not affect the RSS values, we go through the dataset file in small chunks in a streaming fashion during each of the passes over the graph in our preprocessing. The table also reports \emph{preprocessing} time and \emph{total} time, both of which include the time to read the dataset from a file.

\subsection{Orphan-Path Flag Ablation}\label{exp:opf}
To investigate the practical impact of orphan-path flags, we run \texttt{cBK} and \texttt{fcBK} on a subset of the datasets. We capture the impact of the flags as the two implementations are identical except for the orphan-path flags. We report the total and average number of parent-edges traversed during \texttt{findroot} calls, as well as the number of vertices from which we clear orphan-path flags. The results are found in Table~\ref{tab:orphan_ablation}.

\subsection{Results}

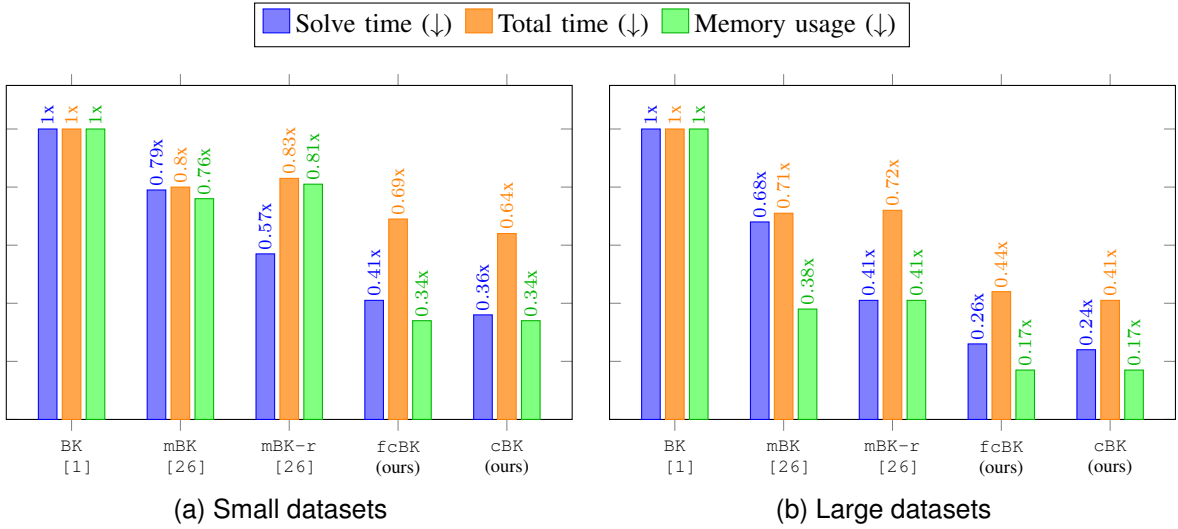
\begin{figure*}[!t]
    \centering
    \pgfplotslegendfromname{sharedlegend}
    \vspace{.5em}
    \subfloat[Small datasets]{
        \begin{tikzpicture}
        \begin{axis}[
            ybar,
            bar width=7pt,
            width=0.5\textwidth,
            height=6cm,
            ymin=0,
            ymax=1.15,
            enlarge x limits=0.15,
            symbolic x coords={BK, mBK, mBK\_r, cBK, cBK\_trad},
            xtick=data,
            xticklabel style={
                font=\scriptsize\ttfamily, 
                align=center, 
            },
            xticklabels={BK\\\cite{boykov2004a}, mBK\\\cite{jensen2023a}, mBK-r\\\cite{jensen2023a},fcBK\\\textnormal{(ours)},cBK\\\textnormal{(ours)}},
            ytick={0.0,0.2,0.4,0.6,0.8,1.0}, 
            yticklabels={},     
            nodes near coords*={\pgfmathprintnumber{\pgfplotspointmeta}x},
            every node near coord/.append style={
                font=\scriptsize,
                rotate=90, 
                anchor=west, 
                xshift=-2pt,
            },
            legend to name=sharedlegend, 
            legend columns=3, 
        ]
        \addplot+[blue, fill=blue!50] coordinates {(BK,1.0) (mBK,0.79) (mBK\_r,0.57) (cBK,0.40) (cBK\_trad,0.35)};
        \addplot+[orange, fill=orange!70] coordinates {(BK,1.0) (mBK,0.8) (mBK\_r,0.83) (cBK,0.64) (cBK\_trad,0.59)};
        \addplot+[green!70!black, fill=green!50] coordinates {(BK,1.0) (mBK,0.76) (mBK\_r,0.81) (cBK,0.34) (cBK\_trad,0.34)};
        \legend{Solve time ($\downarrow$)\;,Total time ($\downarrow$)\;,Graph memory ($\downarrow$)}
        \end{axis}
        \end{tikzpicture}
        \label{fig:small_bar}
    }
    \subfloat[Large datasets]{
        \begin{tikzpicture}
        \begin{axis}[
            ybar,
            bar width=7pt,
            width=0.5\textwidth,
            height=6cm,
            ymin=0,
            ymax=1.15,
            enlarge x limits=0.15,
            symbolic x coords={BK, mBK, mBK\_r, cBK, cBK\_trad},
            xtick=data,
            xticklabel style={
                font=\scriptsize\ttfamily, 
                align=center, 
            },
            xticklabels={BK\\\cite{boykov2004a}, mBK\\\cite{jensen2023a}, mBK-r\\\cite{jensen2023a},fcBK\\\textnormal{(ours)},cBK\\\textnormal{(ours)}},
            ytick={0.0,0.2,0.4,0.6,0.8,1.0}, 
            yticklabels={},     
            nodes near coords*={\pgfmathprintnumber{\pgfplotspointmeta}x},
            every node near coord/.append style={
                font=\scriptsize,
                rotate=90, 
                anchor=west, 
                xshift=-2pt,
            },
        ]
        \addplot+[blue, fill=blue!50] coordinates {(BK,1.0) (mBK,0.68) (mBK\_r,0.41) (cBK,0.25) (cBK\_trad,0.24)};
        \addplot+[orange, fill=orange!70] coordinates {(BK,1.0) (mBK,0.71) (mBK\_r,0.72) (cBK,0.43) (cBK\_trad,0.42)};
        \addplot+[green!70!black, fill=green!50] coordinates {(BK,1.0) (mBK,0.38) (mBK\_r,0.41) (cBK,0.17) (cBK\_trad,0.17)};
        \end{axis}
        \end{tikzpicture}
        \label{fig:large_bar}
    }
    \caption{Bar chart visualizing the average performance of each implementation when measuring the solve time, total time, or graph memory as described in Section~\ref{exp:comp}. Each bar averages the value as a ratio to \texttt{BK} over datasets. The datasets are split into (a) small and (b) large datasets. We do not include the extra-large dataset or the graph instances that consist of multiple subproblems. The full list of running times is found in Table~\ref{tab:running_time_real_world}.} 
    \label{fig:combinedfig_bar}
\end{figure*}

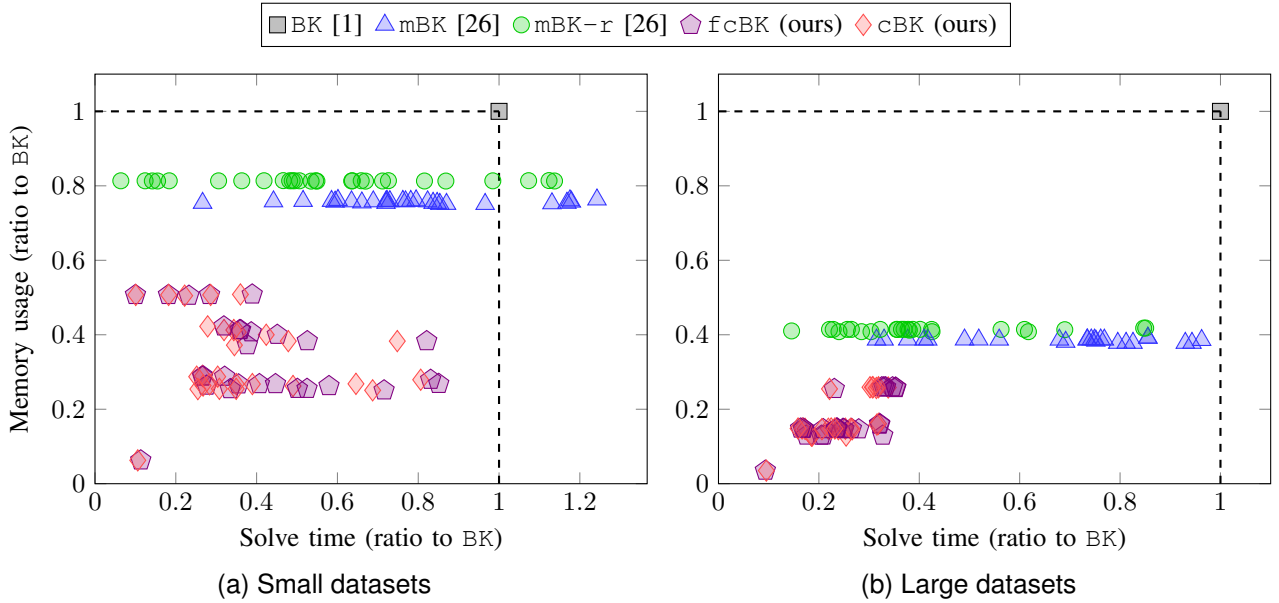
\begin{figure*}[!t]
    \centering
    \pgfplotslegendfromname{relsharedlegend}
    \vspace{.5em}
    \subfloat[Small datasets]{
        \begin{tikzpicture}
        \begin{axis}[
            width=0.49\textwidth,
            height=7cm,
            ymin=0,
            xmin=0,
            ytick={0.0,0.2,0.4,0.6,0.8,1.0}, 
            xlabel={Solve time (ratio to \texttt{BK})},
            ylabel={Graph memory (ratio to \texttt{BK})},
            legend to name=relsharedlegend, 
            legend columns=5, 
        ]
            \addplot [
                scatter,
                only marks,
                point meta=explicit symbolic,
                scatter/classes={
                    BK={mark=square*,black,mark size=3pt,fill opacity=0.25},
                    mBK={mark=triangle*,blue!80!white,mark size=4pt,fill opacity=0.25},
                    mBK_r={mark=*,green!80!black,mark size=3pt,fill opacity=0.25},
                    cBK={mark=pentagon*,violet,mark size=4pt,fill opacity=0.25},
                    cBK_trad={mark=diamond*,red!70!white,mark size=4pt,fill opacity=0.25}
                },
            ] table [x=x, y=y, meta=label] {rel_performance_small.dat};
            \addplot [black, dashed, thick] coordinates {(0,1) (1,1)}; 
            \addplot [black, dashed, thick] coordinates {(1,0) (1,1)}; 
            \legend{\texttt{BK} \cite{boykov2004a}\;, \texttt{mBK} \cite{jensen2023a}\;, \texttt{mBK-r} \cite{jensen2023a}\;, \texttt{fcBK} (ours)\;, \texttt{cBK} (ours)}
        \end{axis}
        \end{tikzpicture}
        \label{fig:small_rel}
    }
    \subfloat[Large datasets]{
        \begin{tikzpicture}
        \begin{axis}[
            width=0.49\textwidth,
            height=7cm,
            ymin=0,
            xmin=0,
            ytick={0.0,0.2,0.4,0.6,0.8,1.0}, 
            xlabel={Solve time (ratio to \texttt{BK})},
        ]
            \addplot [
                scatter,
                only marks,
                point meta=explicit symbolic,
                scatter/classes={
                    BK={mark=square*,black,mark size=3pt,fill opacity=0.25},
                    mBK={mark=triangle*,blue!80!white,mark size=4pt,fill opacity=0.25},
                    mBK_r={mark=*,green!80!black,mark size=3pt,fill opacity=0.25},
                    cBK={mark=pentagon*,violet,mark size=4pt,fill opacity=0.25},
                    cBK_trad={mark=diamond*,red!70!white,mark size=4pt,fill opacity=0.25}
                },
            ] table [x=x, y=y, meta=label] {rel_performance_large.dat};
            \addplot [black, dashed, thick] coordinates {(0,1) (1,1)}; 
            \addplot [black, dashed, thick] coordinates {(1,0) (1,1)}; 
        \end{axis}
        \end{tikzpicture}
        \label{fig:large_rel}
    }
    \caption{Scatterplot visualizing the performance of each implementation across datasets as described in Section~\ref{exp:comp}. Each point maps the solve time and graph memory of an implementation as a ratio to \texttt{BK} for a dataset. The datasets are split into (a) small and (b) large datasets. We do not include the extra-large dataset or the graph instances that consist of multiple subproblems. The full list of running times is found in Table~\ref{tab:running_time_real_world}.} 
    \label{fig:combinedfig_rel}
\end{figure*}

\begin{table*}[p]
\resizebox{\textwidth}{!}{
\begin{threeparttable}
\caption{Median time over five runs for the benchmark experiment of all implementations as described in Section~\ref{exp:comp}. For each dataset, the fastest solve time is \textbf{bold} and the fastest total time is \underline{underlined}.}
\label{tab:running_time_real_world}
\begin{tabular}{lccccrlrlrlrlrl}
\toprule
\multicolumn{5}{c}{Dataset}     & \multicolumn{2}{c}{\texttt{BK} \cite{boykov2004a}} & \multicolumn{2}{c}{\texttt{mBK}\cite{jensen2023a}} & \multicolumn{2}{c}{\texttt{mBK-r} \cite{jensen2023a}} & \multicolumn{2}{c}{\texttt{fcBK} (ours)} & \multicolumn{2}{c}{\texttt{cBK} (ours)}\\
\cmidrule(lr){1-5} \cmidrule(lr){6-7} \cmidrule(lr){8-9} \cmidrule(lr){10-11} \cmidrule(lr){12-13}  \cmidrule(lr){14-15} 
Name & T & $n$ & $m_i^+$ & $m_i$ & \multicolumn{1}{r}{Solve} & \multicolumn{1}{l}{Total} & \multicolumn{1}{r}{Solve} & \multicolumn{1}{l}{Total} & \multicolumn{1}{r}{Solve} & \multicolumn{1}{l}{Total} & \multicolumn{1}{r}{Solve} & \multicolumn{1}{l}{Total} & \multicolumn{1}{r}{Solve} & \multicolumn{1}{l}{Total}\\
\hline
\multicolumn{15}{c}{3D segmentation: voxel-based. From \cite{waterloo}, based on \cite{boykov2001a, boykov2006a, society2003a}.}\\\hline     
*adhead.n26c100 & L & $12.6$ M & $654.3$ M & $324.2$ M & $92.35$ s & $97.98$ s & $85.88$ s & $89.71$ s & $39.28$ s & $67.22$ s & $19.39$ s & $28.24$ s & \boldmath{\textbf{$16.89$ s}} & \underline{$25.64$ s}\\
*adhead.n6c100 & L & $12.6$ M & $151.0$ M & $75.2$ M & $20.79$ s & $22.16$ s & $15.98$ s & $16.88$ s & $12.67$ s & $15.03$ s & \boldmath{\textbf{$5.12$ s}} & \underline{$6.45$ s} & $5.26$ s & $6.59$ s\\
*babyface.n6c100 & L & $5.1$ M & $60.8$ M & $30.2$ M & $7.52$ s & $8.08$ s & $5.63$ s & $5.80$ s & $4.23$ s & $4.98$ s & $1.96$ s & $2.46$ s & \boldmath{\textbf{$1.95$ s}} & \underline{$2.45$ s}\\
*babyface.n26c100 & L & $5.1$ M & $263.2$ M & $129.9$ M & $77.22$ s & $79.47$ s & $72.84$ s & $74.31$ s & $47.68$ s & $58.94$ s & $25.03$ s & $28.62$ s & \boldmath{\textbf{$19.36$ s}} & \underline{$22.95$ s}\\
*bone.n6c100 & L & $7.8$ M & $93.6$ M & $46.5$ M & $5.33$ s & $6.29$ s & $3.62$ s & $4.16$ s & $2.26$ s & $3.73$ s & $1.26$ s & $2.10$ s & \boldmath{\textbf{$1.18$ s}} & \underline{$2.01$ s}\\
*bone.n26c100 & L & $7.8$ M & $405.5$ M & $200.5$ M & $22.07$ s & $25.58$ s & $17.56$ s & $19.81$ s & $5.30$ s & $22.73$ s & \boldmath{\textbf{$3.85$ s}} & \underline{$9.19$ s} & $4.16$ s & $9.47$ s\\
*bone-x.n6c100 & S & $3.9$ M & $46.8$ M & $23.2$ M & $3.94$ s & $4.38$ s & $2.85$ s & $3.12$ s & $2.15$ s & $2.86$ s & \boldmath{\textbf{$1.05$ s}} & \underline{$1.45$ s} & $1.07$ s & $1.48$ s\\
*bone-x.n26c100 & L & $3.9$ M & $202.8$ M & $100.0$ M & $13.74$ s & $15.49$ s & $11.35$ s & $12.49$ s & $4.17$ s & $12.44$ s & $2.80$ s & $5.44$ s & \boldmath{\textbf{$2.54$ s}} & \underline{$5.17$ s}\\
*bone-xyz-xy.n6c10 & S & $245$ K & $2.9$ M & $1.5$ M & $59$ ms & $67$ ms & $30$ ms & $39$ ms & $25$ ms & $47$ ms & $12$ ms & $29$ ms & \boldmath{\textbf{$12$ ms}} & \underline{$29$ ms}\\
*bone-xyz-xy.n6c100 & S & $245$ K & $2.9$ M & $1.5$ M & $65$ ms & $73$ ms & $38$ ms & $47$ ms & $32$ ms & $54$ ms & $17$ ms & $34$ ms & \boldmath{\textbf{$17$ ms}} & \underline{$33$ ms}\\
*liver.n6c100 & S & $4.2$ M & $49.9$ M & $21.8$ M & $8.66$ s & $9.12$ s & $6.77$ s & $6.91$ s & $5.52$ s & $6.08$ s & $2.35$ s & $2.78$ s & \boldmath{\textbf{$2.29$ s}} & \underline{$2.72$ s}\\\hline
*liver.n26c100 & L & $4.2$ M & $216.4$ M & $23.0$ M & $23.11$ s & $24.98$ s & $18.76$ s & $19.98$ s & $6.58$ s & $15.32$ s & \boldmath{\textbf{$2.12$ s}} & $2.75$ s & $2.12$ s & \underline{$2.75$ s}\\
\multicolumn{15}{c}{3D segmentation: oriented MRF. From \cite{jensen2023a}, based on \cite{boykov2004a, grau2006a, reichardt2021a}.}\\\hline
*vessel.orimrf.256 & L & $16.8$ M & $196.5$ M & $96.3$ M & $1.72$ s & $3.48$ s & $0.84$ s & \underline{$1.98$ s} & $0.61$ s & $3.24$ s & $0.42$ s & $2.16$ s & \boldmath{\textbf{$0.40$ s}} & $2.16$ s\\
*vessel.orimrf.512 & L & $134.2$ M & $1.6$ B & $769.6$ M & $11.86$ s & $27.09$ s & $6.15$ s & \underline{$15.08$ s} & $4.59$ s & $24.75$ s & $3.33$ s & $16.69$ s & \boldmath{\textbf{$3.19$ s}} & $16.54$ s\\\hline
\multicolumn{15}{c}{3D U-Net segmentation cleaning. From \cite{jensen2023a}, based on \cite{boykov2004a, grau2006a, reichardt2021a}.}\\\hline
*clean.orimrf.256 & L & $16.8$ M & $198.3$ M & $98.0$ M & $1.18$ s & $2.95$ s & $0.45$ s & \underline{$1.59$ s} & $0.30$ s & $2.95$ s & \boldmath{\textbf{$0.20$ s}} & $1.94$ s & $0.22$ s & $1.95$ s\\
*clean.orimrf.512 & L & $134.2$ M & $1.6$ B & $776.1$ M & $9.24$ s & $24.43$ s & $3.79$ s & \underline{$12.64$ s} & $2.45$ s & $23.40$ s & \boldmath{\textbf{$1.76$ s}} & $15.25$ s & $1.91$ s & $15.41$ s\\
*unet-mrfclean-2 & L & $8.2$ M & $97.6$ M & $48.8$ M & $0.72$ s & $1.63$ s & $0.23$ s & \underline{$0.81$ s} & $0.16$ s & $1.47$ s & $0.12$ s & $0.96$ s & \boldmath{\textbf{$0.12$ s}} & $0.96$ s\\
*unet-mrfclean-3 & L & $15.9$ M & $190.5$ M & $95.2$ M & $1.33$ s & $3.21$ s & $0.44$ s & \underline{$1.62$ s} & $0.29$ s & $2.83$ s & \boldmath{\textbf{$0.22$ s}} & $1.85$ s & $0.22$ s & $1.85$ s\\
*unet-mrfclean-8 & L & $4.9$ M & $58.6$ M & $29.3$ M & $0.54$ s & $1.08$ s & $0.23$ s & \underline{$0.58$ s} & $0.17$ s & $0.96$ s & $0.13$ s & $0.62$ s & \boldmath{\textbf{$0.12$ s}} & $0.62$ s\\\hline
\multicolumn{15}{c}{Surface fitting. From \cite{waterloo}, based on \cite{lempitsky2007a}.}\\\hline
*LB07-bunny-lrg & L & $49.5$ M & $592.9$ M & $296.5$ M & $14.86$ s & $20.53$ s & $8.31$ s & $11.90$ s & $5.29$ s & $13.42$ s & $3.49$ s & $8.53$ s & \boldmath{\textbf{$3.45$ s}} & \underline{$8.48$ s}\\\hline
\multicolumn{15}{c}{3D segmentation: sparse layered graphs (SLG). From \cite{jensen2023a}, based on \cite{jeppesen2020a}.}\\\hline
4Dpipe-small & L & $14.3$ M & $191.4$ M & $191.4$ M & $5.17$ h & $5.17$ h & $4.98$ h & $4.98$ h & $3.57$ h & $3.57$ h & $1.86$ h & $1.86$ h & \boldmath{\textbf{$1.75$ h}} & \underline{$1.75$ h}\\
NT32-tomo3-.raw-3 & L & $7.2$ M & $84.0$ M & $83.8$ M & $22.32$ s & $23.27$ s & $16.40$ s & $16.94$ s & $8.14$ s & $9.64$ s & $7.39$ s & $9.30$ s & \boldmath{\textbf{$6.86$ s}} & \underline{$8.77$ s}\\
NT32-tomo3-.raw-5 & L & $11.1$ M & $128.3$ M & $128.1$ M & $34.83$ s & $36.27$ s & $25.59$ s & $26.41$ s & $12.88$ s & $15.14$ s & $11.34$ s & $14.20$ s & \boldmath{\textbf{$10.59$ s}} & \underline{$13.45$ s}\\
NT32-tomo3-.raw-10 & L & $22.7$ M & $263.5$ M & $262.9$ M & $75.99$ s & $79.00$ s & $56.44$ s & $58.20$ s & $28.70$ s & $33.41$ s & $26.32$ s & $32.22$ s & \boldmath{\textbf{$24.29$ s}} & \underline{$30.18$ s}\\
NT32-tomo3-.raw-30 & L & $67.9$ M & $789.6$ M & $787.9$ M & $336.0$ s & $345.4$ s & $255.8$ s & $261.1$ s & $134.9$ s & $149.1$ s & $110.4$ s & $127.9$ s & \boldmath{\textbf{$103.5$ s}} & \underline{$121.1$ s}\\
NT32-tomo3-.raw-100 & XL & $183.6$ M & $2.1$ B & $2.1$ B & $1134.0$ s & $1158.7$ s & $1039.6$ s & $1060.2$ s & $536.6$ s & $599.1$ s & $404.8$ s & $457.2$ s & \boldmath{\textbf{$378.3$ s}} & \underline{$431.4$ s}\\\hline
\multicolumn{15}{c}{3D segmentation: separating surfaces. From \cite{jensen2023a}, based on \cite{jensen2020a}.}\\\hline
cells.sd3 & L & $13.9$ M & $226.2$ M & $226.2$ M & $66.79$ s & $69.38$ s & $50.21$ s & $52.03$ s & $25.40$ s & $32.75$ s & $23.48$ s & $29.64$ s & \boldmath{\textbf{$21.03$ s}} & \underline{$27.27$ s}\\
foam.set.r60.h210 & S & $2.0$ M & $46.3$ M & $46.3$ M & $1.93$ s & $2.42$ s & $1.27$ s & \underline{$1.43$ s} & \boldmath{\textbf{$0.27$ s}} & $1.69$ s & $0.44$ s & $1.66$ s & $0.43$ s & $1.66$ s\\
foam.set.r160.h210 & L & $15.1$ M & $385.3$ M & $385.3$ M & $18.96$ s & $22.95$ s & $13.11$ s & $15.76$ s & \boldmath{\textbf{$2.76$ s}} & $15.75$ s & $4.30$ s & $15.18$ s & $4.22$ s & \underline{$15.12$ s}\\
*simcells.sd3 & S & $3.6$ M & $47.6$ M & $47.6$ M & $10.47$ s & $10.97$ s & $7.55$ s & $7.75$ s & $5.05$ s & $6.46$ s & $3.99$ s & $4.84$ s & \boldmath{\textbf{$3.75$ s}} & \underline{$4.60$ s}\\\hline
\multicolumn{15}{c}{3D segmentation: Deep LOGISMOS. From \cite{jensen2023a}, based on \cite{quo2018a}.}\\\hline
*deeplogismos.2 & S & $511$ K & $7.2$ M & $7.1$ M & $190$ ms & $0.22$ s & $83$ ms & \underline{$0.11$ s} & \boldmath{\textbf{$12$ ms}} & $0.18$ s & $15$ ms & $0.17$ s & $14$ ms & $0.17$ s\\
*deeplogismos.3 & S & $707$ K & $9.9$ M & $9.9$ M & $104$ ms & $0.15$ s & $85$ ms & \underline{$0.12$ s} & \boldmath{\textbf{$19$ ms}} & $0.21$ s & $26$ ms & $0.22$ s & $27$ ms & $0.22$ s\\
*deeplogismos.7 & S & $989$ K & $13.8$ M & $13.8$ M & $247$ ms & $0.32$ s & $190$ ms & \underline{$0.25$ s} & \boldmath{\textbf{$30$ ms}} & $0.37$ s & $47$ ms & $0.35$ s & $47$ ms & $0.35$ s\\\hline
\multicolumn{15}{c}{Super resolution. From \cite{verma2012a}, based on \cite{rother2007a, freeman2000a}.}\\\hline
super-res-E1 & S & $10$ K & $82$ K & $35$ K & $ 0$ ms & \underline{$ 0$ ms} & $ 0$ ms & $ 1$ ms & $ 0$ ms & $ 1$ ms & $ 0$ ms & $ 1$ ms & \boldmath{\textbf{$ 0$ ms}} & $ 1$ ms\\
super-res-E2 & S & $10$ K & $164$ K & $70$ K & $ 1$ ms & \underline{$ 1$ ms} & $ 1$ ms & $ 2$ ms & $ 1$ ms & $ 2$ ms & $ 1$ ms & $ 2$ ms & \boldmath{\textbf{$ 1$ ms}} & $ 2$ ms\\
superres-graph & S & $43$ K & $1.3$ M & $87$ K & $33$ ms & $36$ ms & $ 8$ ms & $12$ ms & $ 5$ ms & $20$ ms & $ 3$ ms & $ 5$ ms & \boldmath{\textbf{$ 3$ ms}} & \underline{$ 5$ ms}\\\hline
\multicolumn{15}{c}{Texture. From \cite{verma2012a}, based on \cite{rother2007a}.}\\\hline
texture-Cremer & S & $44$ K & $1.4$ M & $695$ K & $0.90$ s & $0.90$ s & $0.76$ s & $0.76$ s & $0.73$ s & $0.75$ s & $0.32$ s & $0.33$ s & \boldmath{\textbf{$0.25$ s}} & \underline{$0.26$ s}\\
texture-OLD-D103 & S & $43$ K & $1.3$ M & $655$ K & $0.38$ s & $0.38$ s & $0.27$ s & $0.27$ s & $0.24$ s & $0.26$ s & $0.15$ s & $0.16$ s & \boldmath{\textbf{$0.13$ s}} & \underline{$0.14$ s}\\
texture-Paper1 & S & $43$ K & $1.3$ M & $655$ K & $0.36$ s & $0.37$ s & $0.30$ s & $0.31$ s & $0.26$ s & $0.28$ s & $0.16$ s & $0.17$ s & \boldmath{\textbf{$0.14$ s}} & \underline{$0.15$ s}\\
texture-Temp & S & $14$ K & $421$ K & $210$ K & $90$ ms & $91$ ms & $105$ ms & $106$ ms & $102$ ms & $105$ ms & $78$ ms & $81$ ms & \boldmath{\textbf{$59$ ms}} & \underline{$62$ ms}\\\hline
\multicolumn{15}{c}{Multi-view. From \cite{waterloo}, based on \cite{boykov2006b, bischof2006a}.}\\\hline
BL06-camel-lrg & L & $18.9$ M & $147.6$ M & $75.4$ M & $65.08$ s & $66.80$ s & $55.67$ s & $56.61$ s & $55.39$ s & $57.72$ s & \boldmath{\textbf{$20.77$ s}} & \underline{$22.55$ s} & $20.81$ s & $22.61$ s\\
BL06-gargoyle-lrg & L & $17.2$ M & $136.5$ M & $68.6$ M & $143.8$ s & $145.4$ s & $123.0$ s & $123.9$ s & $121.8$ s & $124.0$ s & $45.65$ s & $47.33$ s & \boldmath{\textbf{$45.07$ s}} & \underline{$46.75$ s}\\\hline
\multicolumn{15}{c}{Deconvolution. From \cite{verma2012a}, based on \cite{rother2007a}.}\\\hline
graph3x3 & S & $2.0$ K & $87$ K & $43$ K & $ 3$ ms & $ 3$ ms & $ 4$ ms & $ 4$ ms & $ 3$ ms & $ 3$ ms & $ 2$ ms & $ 2$ ms & \boldmath{\textbf{$ 1$ ms}} & \underline{$ 2$ ms}\\
graph5x5 & S & $2.0$ K & $271$ K & $135$ K & $36$ ms & $36$ ms & $35$ ms & $36$ ms & $24$ ms & $26$ ms & $12$ ms & $15$ ms & \boldmath{\textbf{$ 9$ ms}} & \underline{$12$ ms}\\\hline
\multicolumn{15}{c}{Stereo. From \cite{waterloo}, based on \cite{boykov1998a, kolmogorov2001a}.}\\\hline
*BVZ-sawtooth (s) & S & $164$ K & $1.3$ M & $658$ K & $0.38$ s & $0.44$ s & $0.32$ s & $0.39$ s & $0.32$ s & $0.46$ s & $0.22$ s & $0.36$ s & \boldmath{\textbf{$0.22$ s}} & \underline{$0.35$ s}\\
*BVZ-tsukuba (s) & S & $110$ K & $882$ K & $441$ K & $0.18$ s & $0.21$ s & $0.16$ s & $0.20$ s & $0.14$ s & $0.21$ s & $0.11$ s & $0.18$ s & \boldmath{\textbf{$0.11$ s}} & \underline{$0.18$ s}\\
*BVZ-venus (s) & S & $166$ K & $1.3$ M & $663$ K & $0.66$ s & $0.73$ s & $0.60$ s & $0.68$ s & $0.58$ s & $0.74$ s & $0.40$ s & $0.56$ s & \boldmath{\textbf{$0.39$ s}} & \underline{$0.54$ s}\\
KZ2-sawtooth (s) & S & $294$ K & $3.3$ M & $1.7$ M & $1.10$ s & $1.28$ s & $1.02$ s & $1.20$ s & $0.97$ s & $1.37$ s & $0.74$ s & $1.08$ s & \boldmath{\textbf{$0.69$ s}} & \underline{$1.03$ s}\\
KZ2-tsukuba (s) & S & $199$ K & $2.3$ M & $1.1$ M & $0.65$ s & $0.75$ s & $0.61$ s & $0.71$ s & $0.57$ s & $0.79$ s & $0.45$ s & $0.63$ s & \boldmath{\textbf{$0.42$ s}} & \underline{$0.60$ s}\\
KZ2-venus (s) & S & $301$ K & $3.5$ M & $1.7$ M & $1.78$ s & $1.99$ s & $1.61$ s & $1.82$ s & $1.52$ s & $1.97$ s & $1.19$ s & $1.57$ s & \boldmath{\textbf{$1.10$ s}} & \underline{$1.48$ s}\\\hline
\multicolumn{15}{c}{Decision tree field (DTF). From \cite{verma2012a}, based on \cite{nowozin2011a}.}\\\hline
printed-graph1 & S & $20$ K & $2.3$ M & $1.1$ M & $0.48$ s & $0.49$ s & $0.41$ s & $0.42$ s & $0.27$ s & $0.31$ s & $0.24$ s & $0.27$ s & \boldmath{\textbf{$0.16$ s}} & \underline{$0.19$ s}\\
printed-graph16 & S & $11$ K & $1.3$ M & $659$ K & $165$ ms & $168$ ms & $143$ ms & $147$ ms & $88$ ms & $107$ ms & $85$ ms & $102$ ms & \boldmath{\textbf{$50$ ms}} & \underline{$66$ ms}\\\hline
\multicolumn{15}{c}{Graph matching. From \cite{jensen2023a}, based on \cite{hutschenreiter2021a}.}\\\hline
pair1.dd (s) & S & $1.1$ K & $114$ K & $5.1$ K & $ 3$ ms & $ 4$ ms & $ 3$ ms & $ 7$ ms & $ 1$ ms & $ 9$ ms & $ 1$ ms & $ 2$ ms & \boldmath{\textbf{$ 1$ ms}} & \underline{$ 2$ ms}\\\hline
\multicolumn{15}{c}{Mesh segmentation. From \cite{jensen2023a}, based on \cite{liu2015a}.}\\\hline
*bunny.segment & S & $97$ K & $1.2$ M & $877$ K & $91$ ms & $102$ ms & $69$ ms & $72$ ms & $60$ ms & $71$ ms & $40$ ms & $50$ ms & \boldmath{\textbf{$37$ ms}} & \underline{$46$ ms}\\
*bunnybig.segment & S & $2.0$ M & $28.1$ M & $22.1$ M & $1.36$ s & $1.69$ s & $0.86$ s & $1.06$ s & $0.49$ s & $0.95$ s & $0.49$ s & $0.94$ s & \boldmath{\textbf{$0.45$ s}} & \underline{$0.89$ s}\\
*candle.segment & S & $159$ K & $2.1$ M & $1.6$ M & $84$ ms & $105$ ms & $61$ ms & $67$ ms & $41$ ms & $67$ ms & $26$ ms & $45$ ms & \boldmath{\textbf{$23$ ms}} & \underline{$42$ ms}\\
*candlebig.segment & S & $1.3$ M & $11.5$ M & $7.7$ M & $0.46$ s & $0.61$ s & $0.28$ s & $0.36$ s & $0.21$ s & $0.41$ s & $0.17$ s & $0.36$ s & \boldmath{\textbf{$0.16$ s}} & \underline{$0.34$ s}\\
*chair.segment & S & $305$ K & $4.3$ M & $3.4$ M & $0.80$ s & $0.84$ s & $0.55$ s & $0.56$ s & $0.40$ s & $0.47$ s & $0.29$ s & $0.35$ s & \boldmath{\textbf{$0.28$ s}} & \underline{$0.34$ s}\\
*chairbig.segment & S & $3.5$ M & $56.4$ M & $45.9$ M & $2.65$ s & $3.31$ s & $1.57$ s & $1.95$ s & $0.81$ s & \underline{$1.77$ s} & $0.86$ s & $1.96$ s & \boldmath{\textbf{$0.74$ s}} & $1.83$ s\\
*handsmall.segment & S & $15$ K & $155$ K & $108$ K & $ 2$ ms & \underline{$ 2$ ms} & $ 2$ ms & $ 2$ ms & $ 2$ ms & $ 3$ ms & $ 1$ ms & $ 2$ ms & \boldmath{\textbf{$ 1$ ms}} & $ 2$ ms\\
*handbig.segment & S & $248$ K & $2.5$ M & $1.7$ M & $103$ ms & $134$ ms & $82$ ms & $89$ ms & $75$ ms & $114$ ms & $62$ ms & $83$ ms & \boldmath{\textbf{$47$ ms}} & \underline{$64$ ms}\\\bottomrule
\end{tabular}
\begin{tablenotes}
\item[*] All terminal edges (or internal edges) of this dataset have an upper bound on their capacity of 100. This results in $O(m)$ augmenting paths found by the algorithms for this dataset.
\item[T] Size type: A value of 'S' indicates a \emph{small} dataset, so all implementations can build the residual graph $G_f$ and find the maximum $s$-$t$ flow (and minimum $s$-$t$ cut) using a 32-bit compiled program. A value of 'L' indicates a \emph{large} dataset, so implementations (other than \texttt{BK}) can use 32-bit integers to reference vertices and edges. A value of 'XL' indicates an \emph{extra-large} dataset. For this dataset, our implementations use 32-bit relative indices to reference vertices, whereas the other implementations use 64-bit integers.
\item[$n$,$m_i^+$,$m_i$] These columns quantify the number of vertices, internal residual edges (possibly with duplicates, as used by \texttt{BK}, \texttt{mBK}, and \texttt{mBK-r}), and internal residual edges without duplicates (used by our implementations), respectively. 
\item[{(s)}] This dataset consists of multiple subproblems.
\end{tablenotes}
\end{threeparttable}%
}
\end{table*}

\begin{table*}[h]
\resizebox{\textwidth}{!}{
\begin{threeparttable}
\caption{The memory usage in GB for a representative subset of the datasets as described in Section~\ref{exp:peak_mem}. The graph memory is the size of the residual graph (see Table~\ref{tab:spacetable}), whereas the preprocessing memory and solve memory are the measured maximum resident set sizes of our implementations. We have included the graph memory for the existing implementations for reference. The included preprocessing and total time also measure the time to load the dataset into memory. For each dataset, the lowest graph memory and the fastest total time are \textbf{bold}.} 
\label{tab:scalability}
\begin{tabular}{lccccccccccccrl}
\toprule
\multicolumn{6}{c}{Dataset}               & \multicolumn{4}{c}{Graph mem.} & Pre. mem.        & Sol. mem.     & Pre. time & \multicolumn{2}{c}{Tot. time} \\ \cmidrule(rl){1-6} \cmidrule(rl){7-10} \cmidrule(rl){11-11} \cmidrule(rl){12-12} \cmidrule(rl){13-13} \cmidrule(rl){14-15} 
Name & T & $n$ & $m_i^+$ & $m_i$ & $\Delta$ &  \texttt{BK}~\cite{boykov2004a}  & \texttt{mBK}~\cite{jensen2023a} & \texttt{mBK-r}~\cite{jensen2023a} & ours & \multicolumn{3}{c}{\texttt{fcBK}\,/\,\texttt{cBK} (ours)}  & \texttt{fcBK} & \texttt{cBK}      \\ \cmidrule(r){1-6} \cmidrule(rl){7-7}\cmidrule(rl){8-8}\cmidrule(rl){9-9}\cmidrule(rl){10-10} \cmidrule(rl){11-13} \cmidrule(l){14-14} \cmidrule(l){15-15}
*adhead.n26c100.max & L & $12.6$ M & $654.3$ M & $324.2$ M & $26$ & $21.542$ & $8.1411$ & $8.7955$ & \boldmath{\textbf{$2.7950$}} & $2.8495$ & $2.7981$ & $21.22$ s & $42.70$ s & \boldmath{\textbf{$39.65$ s}}\\
*babyface.n26c100.max & L & $5.1$ M & $263.2$ M & $129.9$ M & $26$ & $8.6670$ & $3.2754$ & $3.5387$ & \boldmath{\textbf{$1.1205$}} & $1.1447$ & $1.1236$ & $8.45$ s & $35.42$ s & \boldmath{\textbf{$29.90$ s}}\\
*bone.n6c100.max & L & $7.8$ M & $93.6$ M & $46.5$ M & $6$ & $3.3691$ & $1.3024$ & $1.3960$ & \boldmath{\textbf{$0.4971$}} & $0.5322$ & $0.5003$ & $2.56$ s & \boldmath{\textbf{$4.06$ s}} & $4.11$ s\\
*bone.n26c100.max & L & $7.8$ M & $405.5$ M & $200.5$ M & $26$ & $13.352$ & $5.0458$ & $5.4514$ & \boldmath{\textbf{$1.7288$}} & $1.7640$ & $1.7319$ & $12.64$ s & \boldmath{\textbf{$16.57$ s}} & $16.83$ s\\
*liver.n26c100.max & L & $4.2$ M & $216.4$ M & $23.0$ M & $23$ & $7.1247$ & $2.6926$ & $2.9090$ & \boldmath{\textbf{$0.2509$}} & $0.2716$ & $0.2542$ & $4.21$ s & \boldmath{\textbf{$6.37$ s}} & $6.43$ s\\
*vessel.orimrf.512.bbk.max & L & $134.2$ M & $1.6$ B & $769.6$ M & $6$ & $56.788$ & $21.966$ & $23.540$ & \boldmath{\textbf{$8.3039$}} & $8.8449$ & $8.3068$ & $44.11$ s & $47.51$ s & \boldmath{\textbf{$47.40$ s}}\\
*LB07-bunny-lrg.max & L & $49.5$ M & $592.9$ M & $296.5$ M & $6$ & $21.351$ & $8.2543$ & $8.8473$ & \boldmath{\textbf{$3.1643$}} & $3.3664$ & $3.1673$ & $16.72$ s & $20.95$ s & \boldmath{\textbf{$20.45$ s}}\\
4Dpipe-small.bbk.max & L & $14.3$ M & $191.4$ M & $191.4$ M & $15$ & $6.8118$ & $2.6261$ & $2.8174$ & \boldmath{\textbf{$1.7603$}} & $1.8216$ & $1.7633$ & $21.72$ s & $1.90$ h & \boldmath{\textbf{$1.75$ h}}\\
NT32-tomo3-.raw-100.max & XL & $183.6$ M & $2.1$ B & $2.1$ B & $51$ & $77.049$ & $49.074$ & $51.206$ & \boldmath{\textbf{$19.997$}} & $21.469$ & $20.000$ & $114.8$ s & \boldmath{\textbf{$497.5$ s}} & $502.1$ s\\
foam.set.r160.h210.max & L & $15.1$ M & $385.3$ M & $385.3$ M & $112$ & $13.052$ & $4.9700$ & $5.3553$ & \boldmath{\textbf{$3.3233$}} & $3.3874$ & $3.3264$ & $22.52$ s & $27.06$ s & \boldmath{\textbf{$26.89$ s}}\\
*deeplogismos.7.bbk.max & S & $989$ K & $13.8$ M & $13.8$ M & $324$ & $0.2490$ & $0.1887$ & $0.2025$ & \boldmath{\textbf{$0.1262$}} & $0.1343$ & $0.1294$ & $0.64$ s & \boldmath{\textbf{$0.68$ s}} & $0.69$ s\\
BL06-camel-lrg.max & L & $18.9$ M & $147.6$ M & $75.4$ M & $4$ & $5.6293$ & $2.2055$ & $2.3531$ & \boldmath{\textbf{$0.9055$}} & $0.9852$ & $0.9088$ & $5.11$ s & $26.13$ s & \boldmath{\textbf{$26.02$ s}}\\
graph5x5.max & S & $2.0$ K & $271$ K & $135$ K & $80$ & $0.0044$ & $0.0033$ & $0.0036$ & \boldmath{\textbf{$0.0011$}}& $0.0146$ & $0.0146$ & $ 8$ ms & $21$ ms & \boldmath{\textbf{$17$ ms}}\\
printed-graph16.max & S & $11$ K & $1.3$ M & $659$ K & $62$ & $0.0215$ & $0.0161$ & $0.0174$ & \boldmath{\textbf{$0.0055$}} & $0.0146$ & $0.0146$ & $36$ ms & $0.13$ s & \boldmath{\textbf{$98$ ms}}\\
*chairbig.segment.bbk.max & S & $3.5$ M & $56.4$ M & $45.9$ M & $13$ & $1.0019$ & $0.7585$ & $0.8150$ & \boldmath{\textbf{$0.4234$}} & $0.4417$ & $0.4264$ & $2.16$ s & $3.14$ s & \boldmath{\textbf{$3.04$ s}}\\
Bumblebee (ours) & XL & $803.6$ M & $7.9$ B & $7.9$ B & $10$ & $291.92$ & $186.47$ & $194.39$ & \boldmath{\textbf{$76.195$}}& $82.626$ & $76.198$ & $694.3$ s & $899.9$ s & \boldmath{\textbf{$871.6$ s}}\\
MouseBone (ours) & XL & $932.5$ M & $11.1$ B & $11.1$ B & $56$ & $398.73$ & $253.87$ & $264.93$ & \boldmath{\textbf{$103.41$}} & $110.87$ & $103.41$ & $0.28$ h & $1.05$ h & \boldmath{\textbf{$1.05$ h}}\\
\bottomrule
\end{tabular}
\begin{tablenotes}
\item[*] All terminal edges (or internal edges) of this dataset have an upper bound on their capacity of 100. This results in $O(m)$ augmenting paths found by the algorithms for this dataset.
\item[T] Size type: A value of 'S' indicates a \emph{small} dataset, so all implementations can build the residual graph $G_f$ and find the maximum $s$-$t$ flow (and minimum $s$-$t$ cut) using a 32-bit compiled program. A value of 'L' indicates a \emph{large} dataset, so implementations (other than \texttt{BK}) can use 32-bit integers to reference vertices and edges. The value of 'XL' indicates an \emph{extra-large} dataset. For these datasets, our implementations use 32-bit relative indices to reference vertices, whereas the other implementations use 64-bit integers.
\item[$n$,$m_i^+$,$m_i$,$\Delta$] These columns quantify the number of vertices, internal residual edges (possibly with duplicates, as used by \texttt{BK}, \texttt{mBK}, and \texttt{mBK-r}), internal residual edges without duplicates (used by our implementations), and the maximum number of internal residual edges outgoing from a vertex, respectively. 
\end{tablenotes}
\end{threeparttable}%
}
\end{table*}

\begin{table*}[h]
\resizebox{\textwidth}{!}{
\centering
\begin{threeparttable}
\caption{The orphan-path flag ablation as described in Section~\ref{exp:opf}. The average and total values for 'parent-edge' are the number of parent-edge traversals in \texttt{findroot} calls. The value in opf. cleared is the total number of orphan-path flags that are set and cleared. The solve times are copied from Table~\ref{tab:running_time_real_world} for convenience and the fastest is \textbf{bold}.}
\label{tab:orphan_ablation}
\begin{tabular}{lccccrlrlcrl}
\toprule
\multicolumn{5}{c}{Dataset}            & \multicolumn{2}{c}{Avr. parent-edge} & \multicolumn{2}{c}{Tot. parent-edge}    & \multicolumn{1}{c}{Opf. cleared} & \multicolumn{2}{c}{Solve time} \\ \cmidrule(rl){1-5} \cmidrule(rl){6-7} \cmidrule(rl){8-9} \cmidrule(rl){10-10} \cmidrule(rl){11-12}
Name & T & $n$ & $m_i$ & $\Delta(G_f)$ & \texttt{fcBK} & \texttt{cBK} & \texttt{fcBK} & \texttt{cBK} & \texttt{fcBK}    & \texttt{fcBK}           & \texttt{cBK}           \\ \midrule
*adhead.n26c100.max & L & $12.6$ M & $324.2$ M & $26$ & $2.59$ & $7.77$ & $323.3$ M & $945.8$ M & $61.8$ M & $19.39$ s & \boldmath{\textbf{$16.89$ s}}\\%
*babyface.n26c100.max & L & $5.1$ M & $129.9$ M & $26$ & $3.50$ & $8.04$ & $787.8$ M & $1.8$ B & $157.1$ M & $25.03$ s & \boldmath{\textbf{$19.36$ s}}\\%
*bone.n6c100.max & L & $7.8$ M & $46.5$ M & $6$ & $1.45$ & $3.19$ & $15.9$ M & $31.0$ M & $1.8$ M & $1.26$ s & \boldmath{\textbf{$1.18$ s}}\\%
*bone.n26c100.max & L & $7.8$ M & $200.5$ M & $26$ & $1.32$ & $6.72$ & $28.1$ M & $162.3$ M & $4.3$ M & \boldmath{\textbf{$3.85$ s}} & $4.16$ s\\%
*liver.n26c100.max & L & $4.2$ M & $23.0$ M & $23$ & $5.29$ & $9.11$ & $77.8$ M & $140.8$ M & $10.3$ M & \boldmath{\textbf{$2.12$ s}} & $2.12$ s\\%
*vessel.orimrf.512.bbk.max & L & $134.2$ M & $769.6$ M & $6$ & $0.79$ & $0.86$ & $19.5$ M & $21.2$ M & $12.3$ M & $3.33$ s & \boldmath{\textbf{$3.19$ s}}\\%
*LB07-bunny-lrg.max & L & $49.5$ M & $296.5$ M & $6$ & $1.61$ & $2.05$ & $14.1$ M & $18.1$ M & $4.3$ M & $3.49$ s & \boldmath{\textbf{$3.45$ s}}\\%
4Dpipe-small.bbk.max & L & $14.3$ M & $191.4$ M & $15$ & $13.91$ & $87.12$ & $124.1$ B & $737.2$ B & $39.2$ B & $1.86$ h & \boldmath{\textbf{$1.75$ h}}\\%
NT32-tomo3-.raw-100.max & XL & $183.6$ M & $2.1$ B & $51$ & $2.15$ & $4.86$ & $4.0$ B & $9.1$ B & $1.3$ B & $404.8$ s & \boldmath{\textbf{$378.3$ s}}\\%
foam.set.r160.h210.max & L & $15.1$ M & $385.3$ M & $112$ & $0.14$ & $0.16$ & $6.8$ M & $7.8$ M & $3.6$ M & $4.30$ s & \boldmath{\textbf{$4.22$ s}}\\%
*deeplogismos.7.bbk.max & S & $989$ K & $13.8$ M & $324$ & $0.97$ & $1.37$ & $47$ K & $67$ K & $22$ K & $47$ ms & \boldmath{\textbf{$47$ ms}}\\%
BL06-camel-lrg.max & L & $18.9$ M & $75.4$ M & $4$ & $8.06$ & $8.74$ & $848.3$ M & $936.3$ M & $56.7$ M & \boldmath{\textbf{$20.77$ s}} & $20.81$ s\\%
graph5x5.max & S & $2.0$ K & $135$ K & $80$ & $0.58$ & $1.64$ & $151$ K & $431$ K & $56$ K & $12$ ms & \boldmath{\textbf{$9$ ms}}\\%
printed-graph16.max & S & $11$ K & $659$ K & $62$ & $2.51$ & $3.95$ & $2.6$ M & $3.9$ M & $679$ K & $85$ ms & \boldmath{\textbf{$50$ ms}}\\
*chairbig.segment.bbk.max & S & $3.5$ M & $45.9$ M & $13$ & $2.14$ & $2.55$ & $6.9$ M & $8.3$ M & $1.9$ M & $0.86$ s & \boldmath{\textbf{$0.74$ s}}\\
\bottomrule
\end{tabular}
\begin{tablenotes}
\item[*] All terminal edges (or internal edges) of this dataset have an upper bound on their capacity of 100. This results in $O(m)$ augmenting paths found by the algorithms for this dataset.
\item[T] Size type: A value of 'S', 'L', and 'XL' indicates a \emph{small}, \emph{large}, and \emph{extra-large} dataset, respectively.
\item[$n$,$m_i$,$\Delta(G_f)$] These columns quantify the number of vertices, internal residual edges without duplicates, and the maximum number of internal residual edges outgoing from a vertex, respectively. 
\end{tablenotes}
\end{threeparttable}%
}
\end{table*}

Considering the experiments that compare BK implementations, we observe that \texttt{fcBK} and \texttt{cBK} use less graph memory compared to all existing approaches. We observe that these two implementations average $34\%$ of the graph memory of \texttt{BK} on small datasets and $17\%$ on large datasets, with a spread from $3.5\%$ to $51\%$. We also see that \texttt{mBK} is the closest to our implementations in terms of space, averaging $76\%$ of the graph memory of \texttt{BK} on small datasets and $38\%$ on large datasets (which is more than twice that of \texttt{fcBK} and \texttt{cBK}). The space reduction of our implementations comes mainly from the size reduction of vertices and edges. However, we do also find a large benefit of removing edges from $G$ with $0$ capacity and merging parallel residual edges. On average, these two approaches reduce the number of internal residual edges by $39\%$. However, for some datasets, the reductions in internal residual edges approach $95\%$. The results in Table~\ref{tab:scalability} show that the smaller graph memory also translates into a small peak memory usage in practice for our implementations. The table also shows that the peak size for our additional queues (for orphans and too-large next active relative indices) has little to no impact compared to the graph memory.

Considering the solve time in Table~\ref{tab:running_time_real_world} and \figurename~\ref{fig:combinedfig_bar}, we observe that \texttt{fcBK} and \texttt{cBK} are faster than \texttt{BK} on every dataset, but that \texttt{mBK-r} is the fastest on a few datasets. For total time, we observe that a few datasets have \texttt{BK}, \texttt{mBK}, or \texttt{mBK-r} with the fastest total time. However, on average \texttt{fcBK} and \texttt{cBK} are faster than existing approaches for both solve and total time. Specifically, our implementations average $40\%$ and $34\%$ of the solve time of \texttt{BK} for small datasets as well as $25\%$ and $24\%$ on large datasets, respectively. This is also a significant reduction compared to \texttt{mBK} and \texttt{mBK-r} that average $79\%$ and $57\%$ of the solve time of \texttt{BK} on small datasets as well as $68\%$ and $41\%$ on large datasets, respectively. The per-dataset performance plotted in \figurename~\ref{fig:combinedfig_rel} reveals that the reduction in solve time for \texttt{fcBK} and \texttt{cBK} is consistent across datasets, though more pronounced for large datasets. Comparing Table~\ref{tab:running_time_real_world} and Table~\ref{tab:scalability}, we find that datasets with a large graph memory also have a considerable increase in preprocessing time. However, this is expected as the times in Table~\ref{tab:scalability} include the additional time to load the dataset into memory.

The single extra-large dataset has \texttt{BK} and \texttt{mBK} with a similar solve time of just above $1000$ seconds. On the other hand \texttt{mBK-r}, \texttt{fcBK}, and \texttt{cBK} use $53\%$, $40\%$, and $38\%$ of that time, respectively. The memory usage for this dataset is notable with \texttt{fcBK} and \texttt{cBK} peaking at $21.5$~GB compared to \texttt{BK}, \texttt{mBK}, and \texttt{mBK-r}, with a graph memory of $77$~GB, $49$~GB, and $51$~GB, respectively.

We observe comparable solve times for \texttt{fcBK} and \texttt{cBK} in Table~\ref{tab:running_time_real_world}. We do find \texttt{cBK} faster than \texttt{fcBK} on most datasets. The difference between the two is rarely greater than $10\%$, but consistently tips in the favor of \texttt{cBK}. The largest relative difference for the two implementations is for the \texttt{printed-graph16} dataset, where \texttt{cBK} requires just above half the solve time of \texttt{fcBK}. Looking at the orphan ablation in Table~\ref{tab:orphan_ablation}, we do see a noticeable reduction in the length of the average path searched if we include orphan-path flags, indicating that we do reduce repeated work. However, we also find the average path length searched in each \texttt{findroot} call to be small, even with repeated work. For a few of the datasets, we even find that the total reduction in the number of searched parent edges is smaller than the number of orphan-path flags that we need to clear.

\section{Graph Cuts on Large Volumes} \label{sec:largevolumes}

We create two new graphs from volumetric datasets and compute their minimum $s$-$t$ cuts, highlighting that our implementations can find the minimum $s$-$t$ cut on larger graphs than previously possible using the same hardware. The graphs model surface detection problems as described in \cite{li2006a}, starting from large 3D volumes from two micro-CT (micro‑computed tomography) scans. The graphs required to solve these problems are large enough that \texttt{fcBK} and \texttt{cBK} are the only implementations of BK with small enough peak memory to fit in our $128$~GB machine. We write the graphs to a file to get measurements on peak memory as described in Section~\ref{exp:peak_mem}. The resulting memory usage is found in Table~\ref{tab:scalability}. Note that the dataset files are so large that neither of them fits in memory at once. Still, the graphs satisfy the assumptions of the compact graph representation, showing that these assumptions can be realized for very large datasets. Both graphs use weights that are close to the largest possible weight for the representation, indicated by $\infty$. However, all capacities on edges from $s$ are $O(1)$. So cutting these terminal edges separates $s$ from $t$, ensuring that $|C|=O(n)$. The data and the corresponding graphs for the two datasets are described below.

\subsection{Bumblebee}

\subsubsection{Data and Preprocessing}
The bumblebee dataset is a micro-CT scan of the eye of a bumblebee with dimensions $467\times1055\times1631$. From this scan, we look for the surface of the eye. The raw data is visualized alongside the resulting graph-cut surface in \figurename~\ref{fig:bumblebee}. We process the data by applying a 3D Sobel kernel in the $z$-axis to get extreme values near potential surfaces.

\begin{figure*}[!t]
    \centering
    \subfloat[]{\parbox[b]{.48\textwidth}{
      \label{fig:bumblebee}
      \includegraphics[width=0.46\textwidth]{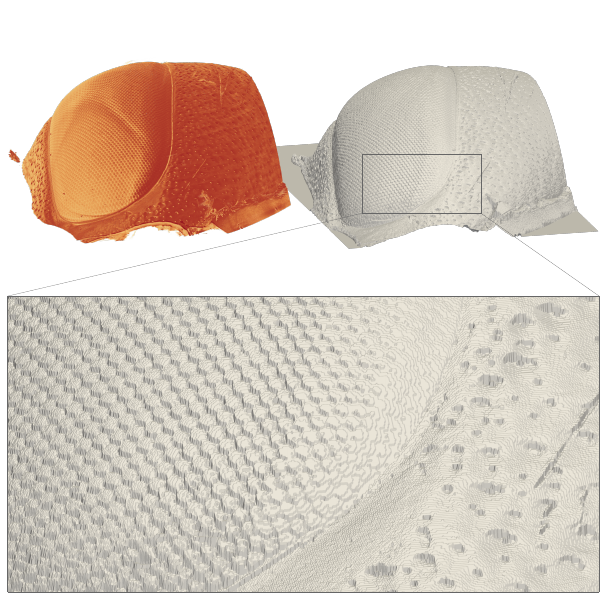}
    }}
    \hfil
    \subfloat[]{
      \label{fig:mousebone}
      \centering
      \includegraphics[width=0.46\textwidth]{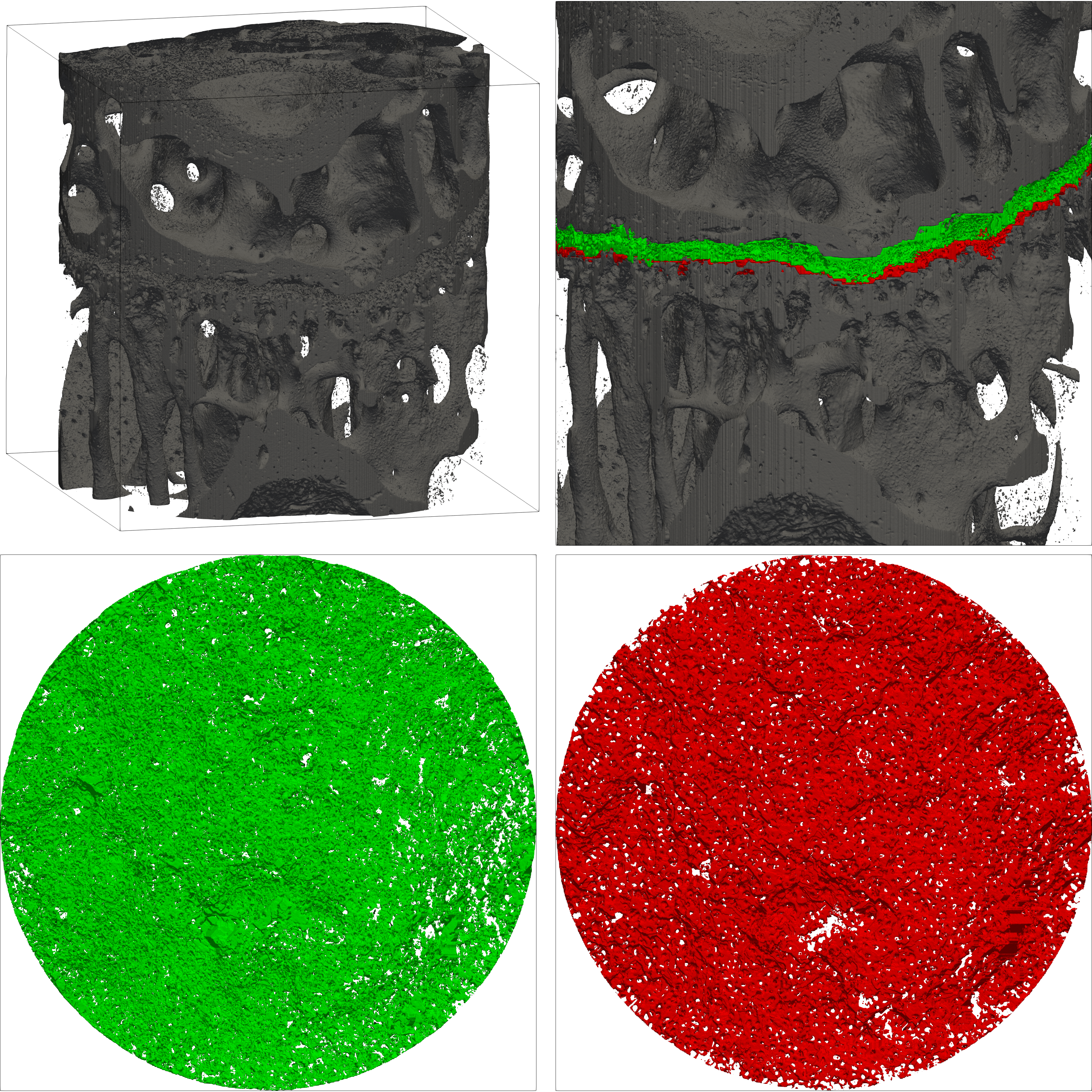}
    }
    \caption{Visualizations of single and multiple surface detection in large volumes as described in Section~\ref{sec:largevolumes}. (a) The eye of a bumblebee (top left). We use \texttt{fcBK} to find the minimum $s$-$t$ cut used to extract a surface (top right). A close-up shows how the graph created from the high-resolution data allows us to capture the fine details of the surface (bottom). (b) The tibia of a mouse. We show the raw data (top left) along with the segmentation of the growth plane (top right). We extract the boundary surfaces of the red and green growth planes (bottom) from a minimum $s$-$t$ cut that we find using \texttt{fcBK}.}
    \label{fig:largevolumesurface}
\end{figure*}

\subsubsection{The Corresponding Graph}
The graph built solves the problem of single surface detection. We create a graph with a vertex corresponding to each voxel in the processed volume, along with the terminals $s$ and $t$. All vertices corresponding to voxels in columns (along $z$) outside of where we expect to find a surface are connected to $t$ with weight $\infty$. We also connect the top $z$-layer to $t$ with an edge of weight $\infty$. For all other vertices, we do the following: Starting with the Sobel-filtered values, we take the voxel-wise difference along the $z$-axis. If negative, the absolute value will correspond to the weight of an edge to $t$, while if positive, that value is used as the weight of an edge from $s$. We also add internal edges to get a single, smooth surface. We do this using a smoothness constraint of $\Delta=8$ between columns (matching notation in \cite{li2006a}) in addition to adding edges directed downwards between neighboring vertices within a column, all of weight $\infty$. 

The graph totals $804$~million vertices (including terminal edges) and $7.9$~billion internal residual edges, making it an extra-large dataset. The compact representation (as used by \texttt{fcBK} or \texttt{cBK}) is the only existing graph representation for the BK algorithm that can fit within the memory bounds of our machine with a peak memory of $83$~GB. In comparison, \texttt{BK} would need $292$~GB to represent the same graph. Our machine requires $12$ minutes to build the graph from the file and $3$ minutes to compute the maximum $s$-$t$ flow.

\subsection{Mouse Bone}\label{sec:mousebone}

\subsubsection{Data and Preprocessing}
The Mouse Bone dataset is a micro-CT scan of the tibia of a mouse. From this scan, we are looking for two regions of the growth plane based on the texture. We use a section of the volume of size $330\times 972\times 972$, as it contains the areas of interest. The raw data is visualized alongside the resulting graph-cut segmentation in \figurename~\ref{fig:mousebone}. From the raw data, we use a neural network to predict the desired classes (can be considered as a closed box). We use the minimum $s$-$t$ cut as a post-processing step, allowing us to gain structure that the neural network is missing and to incorporate the layered structure of the areas. We process the neural network predictions by taking the voxel-wise difference along the $z$-axis to get extreme values near the boundaries of the regions. 

\subsubsection{The Corresponding Graph}
The graph built solves the problem of multiple surface detection for $3$ surfaces. We create a graph with $3$ subgraphs $G_i$ corresponding to the $3$ surfaces, each with a vertex for each voxel in the processed volume, along with the terminals $s$ and $t$. For the top layer of vertices in each subgraph, we connect them to $t$ with weight $\infty$. For the remaining vertices in the $3$ subgraphs, we start with the processed voxels corresponding to the surface of that subgraph. We then take the voxel-wise difference along the $z$-axis. If negative, the absolute value will correspond to the weight of an edge to $t$, while if positive, that value is used as the weight of an edge from $s$. We add edges of weight $\infty$ inside each subgraph (both inside and between columns) to get $3$ smooth surfaces using the smoothness constraint $\Delta=5$  (matching notation in \cite{li2006a}). Between the subgraphs, we add inter-surface edges of weight $\infty$ to ensure a minimum and maximum margin between neighboring surfaces of $1$ and $50$, respectively. 

The graph totals $933$~million vertices (including terminal edges) and $11.1$~billion internal residual edges, making it an extra-large dataset. To ensure that the edges of the graph do not go between vertices with too large of a relative index, we label the vertices using the axis ordering of $x$, $z$, $i$, and then $y$, where $i$ corresponds to the subgraph $G_i$. In other words, we label along the $x$-axis first, then the $xz$ slice, followed by the two other $xz$ slices of $G_2,G_3$, and finally, go to the next $xz$-slice in $G_1$ in the $y$-axis. The compact representation (as used by \texttt{fcBK} or \texttt{cBK}) is the only graph representation for the BK algorithm that can fit within the memory bounds of our machine with a peak memory of $111$~GB. In comparison, \texttt{BK} would need $399$~GB to represent the same graph. Our machine requires $17$ minutes to build the graph from the file and $46$ minutes to find the maximum $s$-$t$ flow.

\section{Discussion and Conclusion}\label{sec:discussion_conclusion}

The tighter running time bound for the BK algorithm of $O(mn|C|)$ reduces a long-standing gap between the theoretical complexity of the algorithm and its empirical running time. Nevertheless, more work is needed to match the strong performance observed for many graphs from computer vision problems. Future work might involve identifying and utilizing parameters in graphs that further improve the worst case running time of the algorithm. We show that a constant upper bound on the weight of each non-terminal edge (similarly for terminal edges) is one such parameter that improves the BK algorithm running time to the polynomial worst case bound of $O(m^2n)$.

The repeated work and theoretical bottleneck in the adoption stage of the BK algorithm does not seem to have a large impact for the considered computer vision datasets. As a result, our fBK algorithm, which alleviates this repeated work using orphan-path flags, has a stronger worst-case guarantee, but not a better practical performance. However, we find that our implementation of the BK algorithm is the fastest available implementation of BK for the considered computer vision datasets. The large improvement in practical running time instead comes from our compact graph representation. This highlights the importance of memory-efficient designs for computer vision implementations of minimum $s$-$t$ cut algorithms.

The memory savings achieved by our compact graph representation are substantial and directly translate into increased scalability. This enables us to process graphs with up to $10^9$ vertices and $10^{10}$ edges on a machine with $128$~GB of memory. This representation is a starting point for further memory-optimized representations in specific use cases and for other minimum $s$-$t$ cut algorithms. Our implementations are publicly available and can be plugged directly into existing and future solutions using minimum $s$-$t$ cuts within computer vision.

\ifthenelse{\useieeetemplate=1}{

\bibliographystyle{IEEEtran}
\bibliography{IEEEabrv,references}
\begin{IEEEbiography}[{\includegraphics[width=1in,height=1.25in,clip,keepaspectratio]{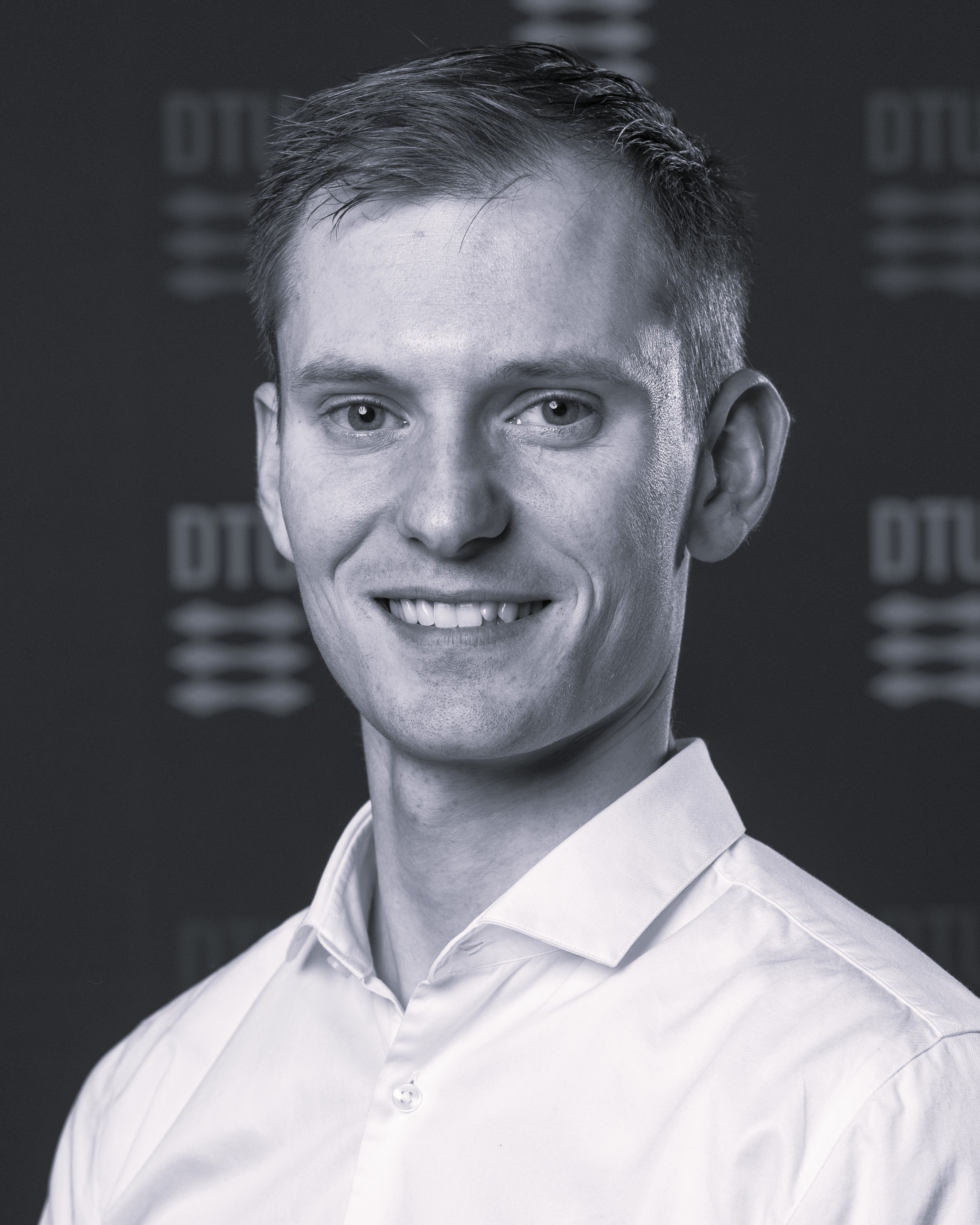}}]{Christian~M.~Mikkelstrup}
received the B.S.E. degree in applied mathematics from the Technical University of Denmark (DTU), Lyngby, Denmark, in 2022 and the M.S.E. degree in human-centered artificial intelligence from DTU in 2024.

From 2022 to 2024, he was a student assistant with TrackMan, where he was engaged in research and development within vision and AI in sports. He is currently a Ph.D. student at the Department of Applied Mathematics and Computer Science, DTU. His research has been concerned with computer vision, deep learning, and graph algorithms. 

Mr. Mikkelstrup is a Ph.D. fellow with the Danish Data Science Academy. For more information, see \url{https://christian.mikkelstrup.info/}.
\end{IEEEbiography}

\begin{IEEEbiography}[{\includegraphics[width=1in,height=1.25in,clip,keepaspectratio]{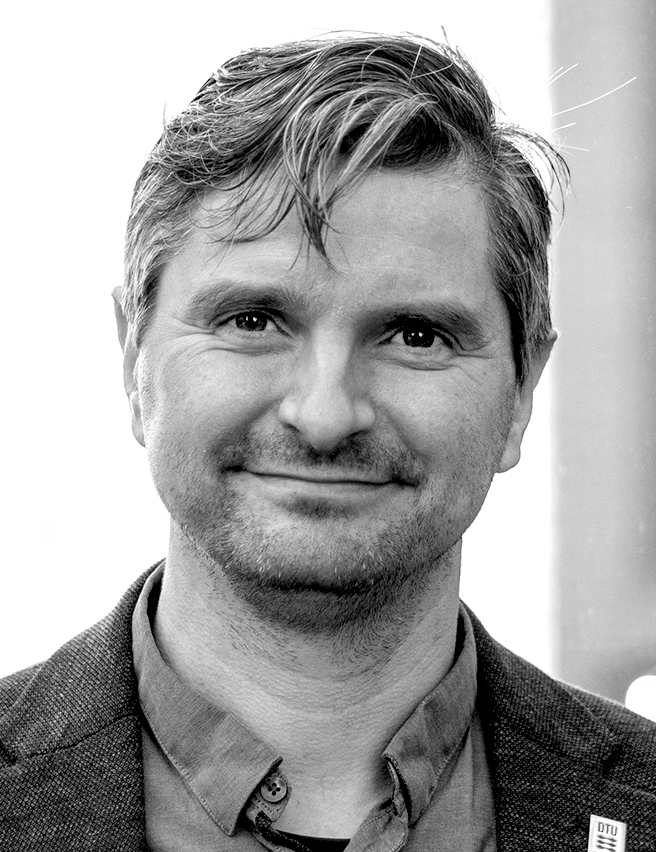}}]{Anders~B.~Dahl}
(Member, IEEE) received the Ph.D. degree from the Technical University of Denmark (DTU), Lyngby, Denmark, in 2009. 

He is currently a professor and heads the Section for Visual Computing at the Department of Applied Mathematics and Computer Science, DTU. His research focuses on machine learning for quantitative analysis of 3D volumetric images.

For more information, see \url{https://orbit.dtu.dk/en/persons/anders-bjorholm-dahl/}.
\end{IEEEbiography}

\begin{IEEEbiography}[{\includegraphics[width=1in,height=1.25in,clip,keepaspectratio]{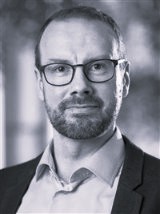}}]{Philip~Bille} received the Ph.D. degree from the IT University of Copenhagen, Copenhagen, Denmark, in 2007. 

He is currently a professor and the head of the Algorithms, Logic, and Graph section at the Department of Applied Mathematics and Computer Science, Technical University of Denmark (DTU), Lyngby, Denmark. His research focuses on searching, indexing, compressing, and manipulating massive data sets. 

Prof. Bille is a member of the Association for Computing Machinery (ACM) and the European Association for Theoretical Computer Science (EATCS). For more information, see \url{https://people.compute.dtu.dk/phbi/}.
\end{IEEEbiography}

\begin{IEEEbiography}[{\includegraphics[width=1in,height=1.25in,clip,keepaspectratio]{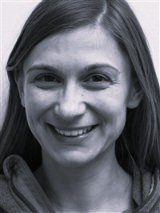}}]{Vedrana~A.~Dahl}
(Member, IEEE) received a degree in mathematics from the University of Zagreb, Zagreb, Croatia, followed by the M.Sc. degree in multimedia technology from the IT University of Copenhagen, Copenhagen, Denmark, and the Ph.D. degree in geometry processing from the Technical University of Denmark (DTU), Lyngby, Denmark. 

She is currently a professor at the Department of Applied Mathematics and Computer Science, DTU. Her research focuses on geometric methods for the analysis of volumetric data, including tomographic reconstruction, volumetric segmentation, deformable meshes, and neural representations. 

For more information, see \url{https://people.compute.dtu.dk/vand/}.
\end{IEEEbiography}

\begin{IEEEbiography}[{\includegraphics[width=1in,height=1.25in,clip,keepaspectratio]{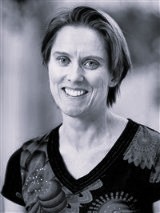}}]{Inge~Li~Gørtz}
 received the Ph.D. degree from the IT University of Copenhagen, Copenhagen, Denmark, in 2005. 

She is currently a professor at the Department of Applied Mathematics and Computer Science, Technical University of Denmark (DTU), Lyngby, Denmark. Her research is concerned with the design and analysis of data structures and algorithms, especially pattern matching, data compression, compressed computation, and approximation algorithms. She is part of the editorial board of Information and Computation.

Prof. Gørtz is a member of the Association for Computing Machinery (ACM) and the European Association for Theoretical Computer Science (EATCS), and she is vice president of the EATCS Assembly. For more information, see \url{https://people.compute.dtu.dk/inge/}.
\end{IEEEbiography}
\vfill
}{
\section*{Acknowledgment}
This work was supported by Danish Data Science Academy, which is funded by the Novo Nordisk Foundation (NNF21SA0069429) and VILLUM FONDEN (40516).

\bibliographystyle{abbrv}
\bibliography{references}
}


\end{document}